\documentclass{article}

\PassOptionsToPackage{numbers}{natbib}

\usepackage[main,final]{neurips_2026}

\usepackage[utf8]{inputenc} 
\usepackage[T1]{fontenc}    
\usepackage{url}            
\usepackage{booktabs}       
\usepackage{amsfonts}       
\usepackage{nicefrac}       
\usepackage{microtype}      
\usepackage{xcolor}         
\usepackage{geometry}
\usepackage{latexsym}
\usepackage{mathrsfs}
\usepackage{amsfonts}
\usepackage{amssymb}
\usepackage{subfigure}    
\usepackage{graphicx}
\usepackage{epstopdf}
\usepackage{float}
\usepackage{multirow}
\usepackage{bm}
\usepackage{amsmath}
\usepackage{amsthm}
\usepackage{color}
\usepackage[subnum]{cases}
\usepackage{rotating}
\usepackage{enumitem}
\usepackage{accents}
\usepackage{algorithm}
\usepackage{algorithmic}
\usepackage[title]{appendix}
\usepackage{mathtools}
\usepackage{graphicx}
\usepackage{subcaption}
\usepackage{booktabs}
\usepackage{url}
\usepackage{multirow}

\makeatletter
\@addtoreset{equation}{section}
\makeatother

\makeatletter
\def \wideubar{\underaccent{{\cc@style\underline{\mskip15mu}}}}
\def \widebar{\accentset{{\cc@style\underline{\mskip10mu}}}}
\makeatother

\graphicspath{{images/}}

\definecolor{blue}{rgb}{0,0,0.9}
\definecolor{red}{rgb}{0.9,0,0}
\definecolor{green}{rgb}{0,0.9,0}
\definecolor{lightgreen}{rgb}{0.1,0.5,0.1}

\newtheorem{proposition}{Proposition}[section]

\newtheorem{definition}{Definition}[section]
\newtheorem{lemma}{Lemma}[section]

\newtheorem{theorem}{Theorem}[section]
\newtheorem{remark}{Remark}[section]

\renewcommand{\cite}{\citep}

\usepackage[colorlinks=true,breaklinks=true,bookmarks=true,urlcolor=blue,
     citecolor=lightgreen,linkcolor=lightgreen,bookmarksopen=false,draft=false]{hyperref}

\title{A Provably Convergent and Practical Algorithm for Gromov--Wasserstein Optimal Transport}

\author{%
  Ling Liang\\
  Department of Mathematics\\
  The University of Tennessee\\
  Knoxville, TN 37916, USA \\
  \texttt{liang.ling@u.nus.edu} \\
  \And
  Lei Yang \thanks{Corresponding author.} \\
  School of Computer Science and Engineering \\
  Guangdong Province Key Laboratory of Computational Science \\
  Sun Yat-sen University \\
  Guangzhou, Guangdong 510275, China\\
  \texttt{yanglei39@mail.sysu.edu.cn} \\
}

\begin{document}

\maketitle

\begin{abstract}
Gromov--Wasserstein optimal transport (GWOT) aligns metric measure spaces by matching their within-domain relational structures, but large-scale GWOT remains challenging because its objective is nonconvex and projection onto the transport polytope is often solved only approximately in practice. This leads to a gap between practical projected-gradient implementations and convergence theory, which typically assumes exact projections. For squared-loss GWOT, we propose an inexact projected-gradient framework with a verifiable feasibility-residual-based inexact condition for the projection subproblem. This condition is directly computable and avoids unknown quantities such as the exact projection point. Under this implementable condition, we prove subsequential convergence to stationary points and, with a mild tolerance-decay condition, convergence of the whole sequence. The resulting method retains the simplicity and sparsity of projected-gradient schemes while providing rigorous convergence guarantees, turning projected-gradient methods into a principled and scalable approach for GWOT with provable reliability.
\end{abstract}

\section{Introduction}

Let $(\mathcal{X},d_{\mathcal{X}})$ and $(\mathcal{Y},d_{\mathcal{Y}})$ be two metric measure spaces equipped with probability vectors ${a}\in\Delta_{n}:=\{{a}\in\mathbb{R}^n_+:\ {1}_n^\top{a}=1\}$ and ${b}\in\Delta_{m}:=\{{b}\in\mathbb{R}^m_+:\ {1}_m^\top{b}=1\}$, where
$1_n$ and $1_m$ denote the all-ones vectors of dimensions $n$ and $m$, respectively. Given intra-domain dissimilarity matrices $C_{\mathcal{X}}\in\mathbb{R}^{n\times n}$ and $C_{\mathcal{Y}}\in\mathbb{R}^{m\times m}$ with, e.g., $(C_{\mathcal{X}})_{ii'}=d_{\mathcal{X}}(x_i,x_{i'})$ and $(C_{\mathcal{Y}})_{jj'}=d_{\mathcal{Y}}(y_j,y_{j'})$, Gromov–Wasserstein optimal transport (GWOT) seeks a \emph{soft coupling} $\Pi\in\mathbb{R}^{n\times m}_+$ that aligns the \emph{relational structures} of the two spaces \cite{Memoli2011}. Unlike classical OT \cite{villani2009optimal}, which relies on a \emph{cross-domain} cost between $x_i$ and $y_j$, GWOT compares \emph{within-domain} relations. It searches for a coupling $\Pi$, under which pairs $(x_i,x_{i'})$ in $\mathcal{X}$ are matched, in aggregate, with pairs $(y_j,y_{j'})$ in $\mathcal{Y}$ having similar dissimilarity. This makes GWOT particularly useful when the two domains live in \emph{incomparable} feature spaces and only their internal metric structures are meaningful. As a result, GWOT has found applications in, e.g., geometry matching and analysis \cite{Memoli2011, PeyreCuturiSolomon2016GW}, graph matching and network alignment \cite{XuLuoZhaCarin2019GWL}, heterogeneous domain adaptation \cite{YanEtAl2018SSOT_HDA}, and single-cell multi-omics alignment \cite{DemetciEtAl2020SCOT}.

Define the transport polytope $\mathcal{U}({a},{b}) :=\left\{\Pi\in\mathbb{R}^{n\times m}_+:\ \Pi{1}_m={a},\ \Pi^\top{1}_n={b}\right\}$. Given a loss function $\ell:\mathbb{R}\times\mathbb{R}\to\mathbb{R}_+$, the GWOT problem is given as follows:
\begin{equation}\label{eq:gworg}
\min_{\Pi\in\mathcal{U}({a},{b})}
\ \mathcal{E}_{\mathrm{GW}}(\Pi)
:=\sum_{i,i'=1}^{n}\sum_{j,j'=1}^{m}
\ell\!\left((C_{\mathcal{X}})_{ii'},(C_{\mathcal{Y}})_{jj'}\right)\,\Pi_{ij}\,\Pi_{i'j'}.
\end{equation}
The objective is generally \emph{nonconvex} due to the bilinear terms $\Pi_{ij}\Pi_{i'j'}$. In this work, we focus on the squared loss $\ell(r,s)=(r-s)^2$, for which the objective admits the decomposition
\begin{equation*}
\mathcal{E}_{\mathrm{GW}}(\Pi)
=
\underbrace{\sum_{i,i'} (C_{\mathcal{X}})_{ii'}^2\,a_i a_{i'}}_{\text{constant}}
+
\underbrace{\sum_{j,j'} (C_{\mathcal{Y}})_{jj'}^2\,b_j b_{j'}}_{\text{constant}}
-
\underbrace{2\sum_{i,i',j,j'} (C_{\mathcal{X}})_{ii'}(C_{\mathcal{Y}})_{jj'}\,\Pi_{ij}\Pi_{i'j'}}_{2\langle C_{\mathcal{X}}\Pi C_{\mathcal{Y}},\,\Pi\rangle}.
\end{equation*}
Therefore, problem \eqref{eq:gworg} is equivalent to
\begin{equation}\label{eq:gw}
\min_{\Pi\in\mathcal{U}({a},{b})}~f(\Pi):= -\langle C_{\mathcal{X}}\Pi C_{\mathcal{Y}},\,\Pi\rangle,
\end{equation}
where $f$ is $L_f$-smooth with $L_f:=2\|C_\mathcal{X}\|_2\|C_\mathcal{Y}\|_2 > 0$.
We highlight that, although we present the analysis for squared-loss GWOT, the proposed framework applies more broadly to GWOT-type problems with Lipschitz-smooth losses, including the fused GWOT \cite{ma2023fused}.

\begin{figure}[t]
\centering
\includegraphics[width=1\linewidth]{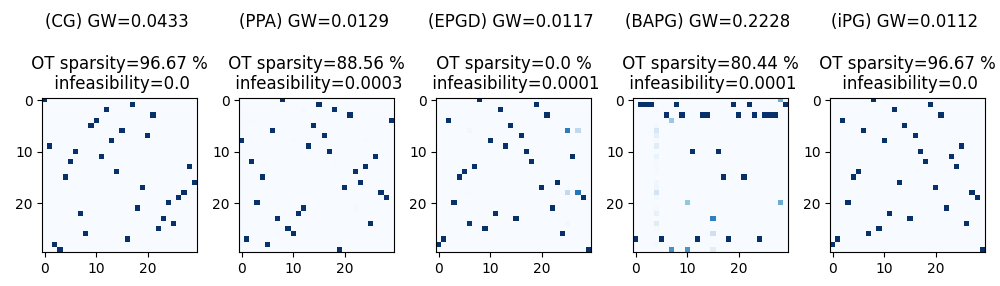}
\caption{GW-distance between samples from 2 Gaussian distributions in 2- and 3-D spaces; see \url{https://pythonot.github.io/auto_examples/index.html}. CG: Conditional Gradient Method \cite{tctf2019optimal}; {PPA}: Entropic Proximal Point Algorithm \cite{XuLuoZhaCarin2019GWL}; EPGD: Entropic Projected Gradient Descent Method \cite{PeyreCuturiSolomon2016GW}; BAPG: Bregman Alternating Projected Gradient {method \cite{li2023convergent}}; iPG: our practical inexact (Euclidean) Projected Gradient method. iPG demonstrates an excellent balance between the GW loss, OT sparsity, and constraint violation.}\label{fig:eg-pot-random}
\end{figure}

Given the structure of problem \eqref{eq:gw}, (Euclidean) projected gradient (PG) methods are among the simplest first-order schemes for solving it: each iteration takes a gradient step and then projects onto the transport polytope. This gives a natural non-entropic approach that preserves the sparse structure expected from solutions of unregularized GWOT problems. In practice, however, exact Euclidean projection onto $\mathcal{U}(a,b)$ can be expensive at scale. Implementations therefore often compute the projection only \textit{approximately}, using an inner iterative solver with a heuristic stopping rule. This creates a gap between practical PG implementations and standard convergence theory, which typically assumes that each projection step is solved exactly. It also helps explain why much recent work on large-scale GWOT has focused on \textit{single-loop} first-order methods; see, e.g., \cite{li2023convergent} and the references therein. By avoiding an inner projection, such methods sidestep the need to specify how accurately the projection subproblem should be solved.

Nevertheless, this leaves open an essential question: \emph{can the simple PG framework be made both practical and rigorous when the projection is computed only approximately?} We answer this question affirmatively. The key idea is to replace the exact projection requirement with a verifiable inexact condition for the projection subproblem. This condition is formulated by the primal feasibility residual associated with the dual iterate, and is therefore directly computable during the inner solve. In particular, it avoids unknown quantities such as the exact projection point. Under this implementable condition, we establish rigorous convergence guarantees for the resulting inexact PG iterates. Consequently, the proposed framework retains the simplicity of PG and provides a rigorous convergence theory for inexact projections controlled by a directly computable stopping criterion. Figure~\ref{fig:eg-pot-random} illustrates the performance of our method compared with representative GWOT solvers.

\paragraph{Contributions.}
Our contributions in this work are summarized as follows.
\begin{itemize}[leftmargin=0.5cm]
\item \textbf{An inexact PG framework with a simple verifiable inexact condition.} We develop an inexact PG (iPG) method for GWOT in which the projection onto the transport polytope $\mathcal U(a,b)$ is solved only \textit{approximately}. The inexact condition is constructed from the primal feasibility residual associated with the dual iterate, making it directly computable during the inner solve while avoiding inaccessible quantities such as the exact projection point.

\item \textbf{Convergence guarantees under inexact computations.} We prove subsequential convergence of the generated inexact iterates to stationary points of the GWOT problem under summable projection errors. We further establish convergence of the whole sequence under an additional mild tolerance-decay condition. These results provide a rigorous foundation for implementing PG for GWOT with computably controlled inexact projections.

\item \textbf{Empirical validation on graph alignment.} We evaluate the proposed method on graph-alignment tasks and compare it with representative GWOT solvers. The results show that iPG produces sparse and nearly feasible transport plans while maintaining competitive objective values, alignment accuracy, and runtime. Additional experiments on paired single-cell data demonstrate competitive alignment quality with sparse, nearly feasible couplings. We also compare against analytical sphere benchmarks and report projection costs and sensitivity to inner accuracy.
\end{itemize}

\subsection{Related Work}

\paragraph{Computational optimal transport.} Optimal transport (OT) provides a geometric framework for comparing probability measures through a cost-minimizing mass transportation problem. In the classical Monge formulation, one seeks a deterministic transport map, which is often difficult to compute. The Kantorovich relaxation instead optimizes over transport plans (couplings), leading to a linear programming problem that forms the foundation of modern OT theory and computation \cite{kantorovich2006translocation}. In data science, computational OT \cite{peyre2019computational} has been driven largely by entropic regularization, which smooths the Kantorovich problem and reduces it to a matrix-scaling problem solvable by Sinkhorn iterations. This has enabled fast and differentiable approximations of Wasserstein-type distances \cite{cuturi2013sinkhorn}. Recent advances in entropic OT have improved scalability through stochastic and semi-discrete or continuous dual methods that require only sampling access to distributions \cite{genevay2016stochastic}, faster algorithms with explicit complexity guarantees such as Greenkhorn and near-linear-time approximation methods \cite{altschuler2017near}, and high-dimensional approximations including low-rank formulations \cite{scetbon2021low} and projection-based approaches such as sliced OT \cite{bonneel2015sliced}. Beyond entropic methods, which generally provide approximate OT solutions, a growing body of work has focused on computing exact OT solutions using splitting methods \cite{liang2024accelerating,lu2024pdot,yang2021fast,zhang2025hot}, entropic proximal point methods \cite{chu2023efficient,wu2025pins,xie2020fast,yang2022bregman}, augmented Lagrangian methods \cite{li2020asymptotically,yang2024corrected,zhu2024ripalm}, and Newton-type methods \cite{hou2024sparse,liu2022multiscale}. These algorithmic advances have further expanded the impact of OT in machine learning, with fruitful applications as mentioned earlier.

\paragraph{Algorithmic frameworks for GWOT.} Algorithmic developments for GWOT can be broadly grouped into several lines. A classical and still widely used approach is conditional-gradient (CG) (a.k.a., Frank--Wolfe-type \cite{FrankWolfe1956}) optimization, which repeatedly linearizes the nonconvex quadratic GW objective over the transport polytope and solves OT-type subproblems, leading to simple iterations but often relatively slow convergence \cite{tctf2019optimal}. A second major line relies on entropic regularization, where the entropic projected gradient descent (EPGD) method combines gradient steps with Sinkhorn-type projections and has become a foundational computational approach for approximate GWOT \cite{PeyreCuturiSolomon2016GW}. Related proximal regularization ideas were further developed in the entropic proximal point algorithm ({PPA}), which solves a sequence of better-conditioned regularized GW subproblems and improves numerical stability \cite{XuLuoZhaCarin2019GWL}. Beyond these basic first-order schemes, several scalable variants have been proposed for graph matching and partitioning. A representative example is the S-GWL framework, which exploits multi-scale and recursive strategies to extend GW-based matching to larger graphs \cite{xlc2019Scalable}. More recent work has pursued stronger algorithmic simplifications, broader modeling variants, and sharper guarantees. The Bregman alternating projected gradient
(BAPG) method gives a provable single-loop first-order scheme for a relaxed GW formulation \cite{li2023convergent}. Semi-relaxed and unbalanced GW formulations enlarge the modeling scope and lead to more flexible algorithms for graphs and measures with mass variation or outliers
\cite{sejourne2021unbalanced,vincent2022semirelaxed}. Low-rank methods can dramatically reduce computational complexity and yield near-linear or linear-time approximations in favorable regimes \cite{scetbon2022linear}, while semidefinite relaxations provide a complementary direction aimed at stronger lower bounds and, in some cases, certificates of global optimality beyond local stationarity \cite{chen2024sdpGW}. In this landscape, our practical inexact Euclidean projected-gradient method belongs to the class of non-entropic first-order methods. Its distinguishing feature is that it works directly over the original transport polytope while allowing the projection subproblem to be solved only approximately under a simple and verifiable inexact condition.

\paragraph{Notation.} We use $\mathbb{R}^n$ ($\mathbb{R}_{+}^n$) and $\mathbb{R}^{n \times m}$ ($\mathbb{R}_{+}^{n \times m}$) to denote the sets of $n$-dimensional real (non-negative) vectors and $n \times m$ real (non-negative) matrices, respectively. For a vector $x\in\mathbb{R}^n$, $\|x\|_1$ and $\|x\|_2$ denote its $\ell_1$ and $\ell_2$ norms, respectively. For a matrix $X\in\mathbb{R}^{n\times m}$, $\|X\|_2$ and $\|X\|_F$ denote its spectral and Frobenius norms, respectively, and we write $\|X\|_1=\sum_{ij}|X_{ij}|$. For a proper closed convex function
$h:\mathbb{R}^{n\times m}\to(-\infty,\infty]$ and a scalar $\delta\ge0$, the
$\delta$-subdifferential of $h$ at $X\in\mathrm{dom}\,h$ is defined by
$\partial_\delta h(X):=\left\{
D\in\mathbb{R}^{n\times m}:
h(Y)\geq h(X)+\langle D,Y-X\rangle-\delta,~~
\forall\,Y\in\mathbb{R}^{n\times m}
\right\}$.
When $\delta=0$, $\partial_\delta h$ reduces to the classical convex
subdifferential $\partial h$. For a closed convex set
$\mathcal{C}\subseteq\mathbb{R}^{n\times m}$, its indicator function is denoted
by $\delta_{\mathcal{C}}$, where $\delta_{\mathcal{C}}(X)=0$ if
$X\in\mathcal{C}$ and $\delta_{\mathcal{C}}(X)=+\infty$ otherwise. We use
$\mathrm{Proj}_{\mathcal{C}}(X)$ to denote the Euclidean projection of $X$ onto
$\mathcal{C}$, and $\mathrm{dist}(X, \mathcal{C})$ to denote the distance from $X$ to $\mathcal{C}$, i.e., $\mathrm{dist}(X, \mathcal{C}) := \inf_{Y\in\mathcal{C}}\|Y-X\|$.

\section{A Provably Convergent Inexact Projected Gradient Framework}\label{sec-ipg}

For simplicity, we denote the linear mapping $\mathcal{A}:\mathbb{R}^{n\times m}\to\mathbb{R}^{n+m}$ and the vector $r \in \mathbb{R}^{n+m}$ as $\mathcal{A}\Pi\;:=\;
[\Pi1_m; \,\Pi^\top1_n]\in\mathbb{R}^{n+m}$ and $r:= [a;\,b]\in \mathbb{R}^{n+m}$, respectively. The adjoint mapping of $\mathcal{A}$ is denoted as $\mathcal{A}^*$. In this section, building on the classical projected-gradient (PG) framework for solving the GWOT problem \eqref{eq:gw}, we develop a novel practical \emph{inexact projected gradient} (iPG) method. Our key contribution is a novel feasibility-residual-based inexact condition that is \emph{simple to check in practice} and \emph{truly implementable}, avoiding unknown quantities and hard-to-verify requirements, yet still yields strong convergence guarantees comparable to those of the exact PG method. The complete algorithmic framework is presented in Algorithm \ref{algo-iPG}.

\begin{algorithm}[htb!]
\caption{iPG for GWOT.}\label{algo-iPG}
\textbf{Input:} Let $\{\varepsilon_k\}_{k=0}^{\infty}$ be a sequence of nonnegative scalars satisfying $\sum_{k=0}^\infty \varepsilon_k < \infty$, and let $\gamma > L_f$ be the inverse step size. Choose $\Pi^0 \in\mathbb{R}_+^{n\times m}$ arbitrarily.
Set $k=0$.   \\
\textbf{while} a termination criterion is not met, \textbf{do} \vspace{-1mm}
\begin{itemize}[leftmargin=2cm]
\item[\textbf{Step 1}.] Find a vector $y^{k+1}$ such that the matrix
    \begin{equation}\label{defHk}
    \Pi^{k+1}:= \mathrm{Proj}_{\mathbb{R}_+^{n\times m}}\left(\mathcal{A}^*y^{k+1} + \Pi^k - \tfrac{1}{\gamma}\nabla f(\Pi^k)\right) \in \mathbb{R}_+^{n\times m}
    \end{equation}
    is an approximate solution of the projection subproblem
    \begin{equation}\label{subpro-iPGM}
    \min\limits_{\Pi\in \mathcal{U}(a,b)}~\langle \nabla f(\Pi^k), \,\Pi-\Pi^k\rangle + \frac{\gamma}{2}\|\Pi-\Pi^k\|_F^2,
    \end{equation}
    satisfying the following inexact condition:
    \begin{equation}\label{eq-inexact-ipg-subpro}
    \|\mathcal{A}\Pi^{k+1} - r\|_2 \leq \frac{\varepsilon_k}{1+\|y^{k+1}\|_2}.
    \end{equation}

\item[\textbf{Step 2}.] Set $k = k+1$ and go to \textbf{Step 1}. \vspace{-1mm}
\end{itemize}
\textbf{end while}  \\
\textbf{Output}: $\Pi^k$ \vspace{0.5mm}
\end{algorithm}

\paragraph{Well-definiteness of Algorithm~\ref{algo-iPG}.}
We first show that the above algorithm is well-defined. Specifically, fix any $k\ge 0$ and any prescribed tolerance $\varepsilon_k>0$. We claim that there always exists a dual vector $y^{k+1}\in\mathbb{R}^{n+m}$ such that
\begin{equation}\label{eq:well_defined_goal}
\Big\|\mathcal{A}\,\mathrm{Proj}_{\mathbb{R}_+^{n\times m}}\Big(\mathcal{A}^*y^{k+1}+\Pi^k-\tfrac{1}{\gamma}\nabla f(\Pi^k)\Big)-r\Big\|_2
\leq\frac{\varepsilon_k}{1+\|y^{k+1}\|_2}.
\end{equation}
To justify~\eqref{eq:well_defined_goal}, consider the dual formulation of the projection subproblem~\eqref{subpro-iPGM}, which is the unconstrained convex minimization problem
\begin{equation}\label{eq-dual-proj}
\min_{y\in\mathbb{R}^{n+m}}\;
g_k(y):=\frac{1}{2}\Big\|\mathrm{Proj}_{\mathbb{R}_+^{n\times m}}\Big(\mathcal{A}^*y+\Pi^k-\tfrac{1}{\gamma}\nabla f(\Pi^k)\Big)\Big\|_F^2-\langle r,y\rangle.
\end{equation}
A direct calculation gives $ \nabla g_k(y) = \mathcal{A}\,\mathrm{Proj}_{\mathbb{R}_+^{n\times m}}\Big(\mathcal{A}^*y+\Pi^k-\tfrac{1}{\gamma}\nabla f(\Pi^k)\Big)-r$.
Let $y^{k,*}$ be any minimizer of~\eqref{eq-dual-proj}. By the first-order optimality condition, we have that $\nabla g_k(y^{k,*})=0$ and the associated primal solution is $\Pi^{k,*} = \mathrm{Proj}_{\mathbb{R}_+^{n\times m}}\Big(\mathcal{A}^*y^{k,*}+\Pi^k-\tfrac{1}{\gamma}\nabla f(\Pi^k)\Big)$,
which is the (unique) optimal solution of the primal projection subproblem \eqref{subpro-iPGM}. Moreover, since $g_k$ is convex and continuously differentiable, problem \eqref{eq-dual-proj} can be solved by standard algorithms for smooth convex optimization, such as accelerated gradient methods
\cite{beck2009fast,nesterov1983method} or Newton-type methods
\cite{liang2021inexact,yang2023inexact,zhao2010newton}. In particular, one can
apply a convergent inner solver to generate a sequence $\{y^{k,t}\}_{t\ge 0}$
satisfying $y^{k,t}\to y^{k,*}$ as $t\to\infty$ for some minimizer $y^{k,*}$.
By the continuity of $\nabla g_k$ and the optimality condition
$\nabla g_k(y^{k,*})=0$, we obtain
\begin{equation*}
\|\nabla g_k(y^{k,t})\|_2
=\Big\|\mathcal{A}\,\mathrm{Proj}_{\mathbb{R}_+^{n\times m}}\!\Big(
\mathcal{A}^*y^{k,t}+\Pi^k-\tfrac{1}{\gamma}\nabla f(\Pi^k)\Big)-r\Big\|_2
\longrightarrow 0
\quad\text{as}~~t\to\infty.
\end{equation*}
Consequently, if $\varepsilon_k>0$, there exists an inner iteration index $t$
large enough such that $y^{k+1}:=y^{k,t}$ satisfies~\eqref{eq:well_defined_goal};
if $\varepsilon_k=0$, one may take $y^{k+1}=y^{k,*}$. Thus, the inexact
condition~\eqref{eq-inexact-ipg-subpro} is attainable at every outer iteration
$k$, and Algorithm~\ref{algo-iPG} is well defined.

\begin{remark}[\textbf{Comments on scaling by $1+\|y^{k+1}\|_2$ in \eqref{eq-inexact-ipg-subpro}}]\label{rem:why_y_scaling}
The factor $(1+\|y^{k+1}\|_2)^{-1}$ is included to make \eqref{eq-inexact-ipg-subpro} robust
to potentially large dual iterates. In general, without additional regularity
assumptions, there is no guarantee that the sequence
$\{\|y^{k+1}\|_2\}_{k\ge 0}$ remains uniformly bounded. Since the
analysis involves products between the dual iterate and the primal feasibility
residual, the normalization ensures that these products remain controlled even
when $\|y^{k+1}\|_2$ is large. At the same time, the condition still forces the
primal feasibility residual to vanish as $\varepsilon_k\to0$.
\end{remark}

We next analyze the convergence behavior of the proposed algorithm. We first
establish \emph{subsequential convergence} to stationary points and then prove
\emph{full sequential convergence} of the whole sequence. Proofs of all theoretical results are deferred to the Appendices. For notational convenience, let $\varepsilon_{-1}:=\|\mathcal{A}\Pi^0-r\|_2$ and $E:=\sum_{k=-1}^{\infty}\varepsilon_k<\infty$. These two quantities are used
only in the analysis and are not needed in practical implementations.

\begin{theorem}\label{thm-subsequential-convergence}
Suppose that $\sum_{k=0}^{\infty}\varepsilon_k<\infty$. Let $\{\Pi^k\}_{k=0}^{\infty}$ be the sequence generated by Algorithm \ref{algo-iPG}. Then, there exists a shadow sequence $\{\widetilde{\Pi}^{k}\}_{k=0}^{\infty}$ satisfying $\widetilde{\Pi}^{k}\in \mathcal{U}(a,b)$ and $\|\widetilde{\Pi}^{k} - \Pi^{k}\|_F\leq 2\sqrt{m+n}\,\|\mathcal{A}\Pi^k-r\|_2$ such that the following statements hold.
\begin{enumerate}[leftmargin=0.5cm]
\item \textbf{(Boundedness)} Both sequences $\{\Pi^k\}$ and $\{\widetilde{\Pi}^{k}\}$ are bounded;

\item \textbf{(Approximate sufficient descent property along $\{\widetilde{\Pi}^{k}\}$)} For any $k\geq0$, it holds that
    \begin{equation}\label{eq-sufficient-descent}
    f(\widetilde{\Pi}^{k+1}) - f(\widetilde{\Pi}^k) \leq - \frac{\gamma-L_f}{2}\|\widetilde{\Pi}^{k+1} - \Pi^k\|_F^2 - \frac{\gamma}{2}\|\widetilde{\Pi}^{k} - \Pi^{k+1}\|_F^2 + \xi_k,
    \end{equation}
    where $\xi_k:=4\gamma(m+n)(\varepsilon_{k-1}^2 + \varepsilon_{k}^2) + (6E+4)2\gamma(m+n)\varepsilon_k$;

\item \textbf{(Shared accumulation points)} The limit $f^*:=\lim\limits_{k\to \infty}f(\widetilde{\Pi}^k)$ exists, and $\lim\limits_{k\to\infty}\; \|\widetilde{\Pi}^{k+1} - \Pi^k\|_F = \lim\limits_{k\to\infty}\; \|\widetilde{\Pi}^{k} - \Pi^{k+1}\|_F = 0$. Moreover, any accumulation point of $\{\Pi^k\}$ is an accumulation point of $\{\widetilde{\Pi}^{k}\}$, and vice versa;

\item \textbf{(Subsequential convergence)} Any accumulation point $\Pi^*$ of $\{\Pi^k\}$ or $\{\widetilde{\Pi}^{k}\}$ is a stationary point of the GWOT problem \eqref{eq:gw}, i.e., $0 \in \partial \delta_{\mathcal{U}(a,b)}(\Pi^{*}) + \nabla f(\Pi^*)$. Moreover, $f(\Pi^*) = f^*$ for any accumulation point  $\Pi^*$ of $\{\Pi^k\}$ or $\{\widetilde{\Pi}^{k}\}$.
\end{enumerate}
\end{theorem}

Theorem~\ref{thm-subsequential-convergence} establishes subsequential convergence
by constructing a \emph{feasible} shadow sequence $\{\widetilde{\Pi}^k\}$ close
to the possibly \emph{infeasible} iterates $\{\Pi^k\}$. This is sufficient for
proving that all accumulation points are stationary. However, to prove convergence
of the whole sequence, one needs a stronger mechanism: a descent sequence that is
compatible with the projected gradient geometry and the Kurdyka--{\L}ojasiewicz
(KL) framework (see \cite{abs2013convergence,bdl2007the,bst2014proximal} for more details). The auxiliary sequence $\{\widetilde{\Pi}^k\}$ is not uniquely
defined and is not generated by the exact projected gradient step, so it is not
the most convenient object for a full sequential convergence argument.

For this reason, we introduce a second shadow sequence, namely, the exact
projection points
\begin{equation}\label{defHstar}
\Pi^{k,*}
:=\arg\min_{\Pi\in\mathcal{U}(a,b)}
\frac{1}{2}\left\|\Pi-\left(\Pi^k-\frac{1}{\gamma}\nabla f(\Pi^k)\right)\right\|_F^2.
\end{equation}
The sequence $\{\Pi^{k,*}\}$ plays the role of the ideal exact PG iterate. If the
computed inexact iterate $\Pi^{k+1}$ can be controlled relative to
$\Pi^{k,*}$, then the implementable residual condition can be converted into an
approximate sufficient descent property for $\{\Pi^{k,*}\}$. This is the key
ingredient needed to combine the inexact PG recursion with the KL property and
establish full sequential convergence.

The following proposition provides precisely this bridge. It shows that the
primal feasibility residual $\|\mathcal{A}\Pi^{k+1}-r\|_2$ controls the distance from the computed inexact projection $\Pi^{k+1}$ to the exact projection $\Pi^{k,*}$, \textit{uniformly} in $k$.

\begin{proposition}[Error bound for the inexact projection step]\label{prop:projection-error-bound}
Let $\{\Pi^k\}$ be generated by Algorithm~\ref{algo-iPG}, and let $\{\Pi^{k,*}\}$ be defined by \eqref{defHstar}. Then, there exists a constant $\omega>0$, independent of $k$, such that $\|\Pi^{k+1}-\Pi^{k,*}\|_F \leq \omega\|\mathcal{A}\Pi^{k+1}-r\|_2, \; \forall\,k\geq0$.
Consequently, if the inexact condition \eqref{eq-inexact-ipg-subpro} holds, then $\|\Pi^{k+1}-\Pi^{k,*}\|_F \leq \omega\varepsilon_k, \; \forall\,k\geq0.$
\end{proposition}

\begin{theorem}\label{thm-full-convergence}
Let $\{\Pi^k\}$ be generated by Algorithm~\ref{algo-iPG}, and let
$\{\Pi^{k,*}\}$ be defined by \eqref{defHstar}. Then,
\begin{enumerate}[leftmargin=0.5cm]
\item \textbf{(Approximate sufficient descent property along $\{\Pi^{k,*}\}$)} For any $k\geq1$, it holds that
        \begin{equation*}
        f(\Pi^{k,*}) - f(\Pi^{k-1,*}) \leq -\frac{\gamma-L_f}{2}\Big(\|\Pi^{k,*} - \Pi^{k-1,*}\|_F^2 + \|\Pi^{k,*} - \Pi^{k}\|_F^2\Big)
        + \gamma\omega^2\varepsilon_{k-1}^2.
        \end{equation*}

\item \textbf{(Shared accumulation points)} The sequence $\{\Pi^{k,*}\}$ is bounded and has the same set of accumulation points as $\{\Pi^k\}$. Consequently, by Theorem~\ref{thm-subsequential-convergence}, any accumulation point of $\{\Pi^{k,*}\}$ is a stationary point of the GWOT problem~\eqref{eq:gw}.

\item \textbf{(Full sequential convergence)} Let $\tau > 1$ be an arbitrarily given rational constant. Suppose that the tolerance parameter sequence $\{\varepsilon_k\}_{k=0}^{\infty}$ is chosen such that
    \begin{equation}\label{condontolpara}
    \sum_{k=0}^\infty \varepsilon_k < \infty, \quad
    \sum_{k=0}^\infty\left(\sum_{i=k}^\infty\varepsilon_i^2\right)^{1 - \frac{1}{\tau}} < \infty.
    \end{equation}
    Then, the whole sequence $\{\Pi^{k}\}$ converges to a stationary point of the GWOT problem \eqref{eq:gw}.
\end{enumerate}
\end{theorem}

\begin{remark}[A simple choice of $\{\varepsilon_k\}$]
Let $\tau>1$. To ensure \eqref{condontolpara}, one can set
$\varepsilon_k=(k+1)^{-\alpha}$ with $\alpha>1$ and $(2\alpha-1)\left(1-\frac{1}{\tau}\right)>1$. Indeed, $\sum_{k=0}^\infty\varepsilon_k<\infty$ follows from $\alpha>1$, and for every $k\geq1$, $\sum_{i=k}^\infty\varepsilon_i^2
\leq \int_k^\infty t^{-2\alpha}\,dt
= \frac{k^{-(2\alpha-1)}}{2\alpha-1}$. Since the $k=0$ term is finite, $\sum_{k=0}^\infty(\sum_{i=k}^\infty\varepsilon_i^2)^{1-1/\tau}<\infty$ whenever $(2\alpha-1)(1-1/\tau)>1$. Equivalently,
$\alpha>\frac{1}{2}+\frac{\tau}{2(\tau-1)}$, whose right-hand side decreases to $1$ as $\tau\to\infty$. Hence, by taking $\tau$ sufficiently large, one may choose a decay exponent $\alpha$ only slightly larger than $1$ in practice, as in our numerical experiments.
\end{remark}

\section{Numerical Experiments}\label{sec:numerical-graph-alignment}

We compare iPG with representative GWOT solvers on graph and paired single-cell alignment, assessing objective quality, efficiency, feasibility, sparsity, and alignment accuracy. We also examine projection costs, sensitivity to inner accuracy, and analytical sphere benchmarks. All numerical experiments were performed on a PC equipped with an Intel Core i7-13700K processor and 64~GB of RAM, running 64-bit Windows~11 with Python~3.12.9, NumPy~2.1.2, SciPy~1.15.2, and POT~0.9.6.


\subsection{Graph alignment}

\paragraph{Graph alignment problem and GWOT formulation.} Graph alignment \cite{XuLuoZhaCarin2019GWL} aims to identify correspondences between the node sets of two graphs that represent the same or similar underlying relational structure. Let $G_1 = (V_1,E_1)$ and $G_2 = (V_2,E_2)$ be two undirected graphs with $|V_1| = n_1$ and $|V_2| = n_2$, respectively. In our experiments, we set $n_1=n_2=n$. Here, we assume that the two graphs are from a common latent graph up to a hidden permutation, possibly with additional structural perturbations. The goal is to recover the unknown node matching between $V_1$ and $V_2$.

A key difficulty in graph alignment is that node labels are arbitrary, so direct comparison of adjacency matrices is not meaningful. Instead, one compares the \emph{intrinsic geometry} of the two graphs. To this end, we construct their intra-graph cost matrices $C_1 \in \mathbb{R}^{n_1\times n_1}$ and $C_2 \in \mathbb{R}^{n_2\times n_2}$, where $(C_1)_{ii'}$ and $(C_2)_{jj'}$ encode pairwise shortest-path distances between nodes within each graph.\footnote{Without loss of generality, we assume that both graphs are connected.} These matrices summarize the internal metric structure of the graphs and are invariant to relabeling.  Given node mass distributions $p, q$, the graph alignment problem can be relaxed into a GWOT problem. In our experiments, we use the standard squared loss and set $p = \tfrac{1}{n}1_{n},\; q = \tfrac{1}{n}1_{n}$. {The transport matrix $\Pi$} obtained from solving the GWOT then plays the role of a soft matching between the node sets of the two graphs: a large value of $\Pi_{ij}$ indicates that node $i\in V_1$ is strongly associated with node $j\in V_2$ with the hope that a good coupling should map pairs of nodes in $G_1$ to pairs of nodes in $G_2$ with similar relational distances. Therefore, even when node identities are completely scrambled, the GWOT objective may recover alignment information purely from structural consistency.

\paragraph{Graph generation.} We generate graph alignment instances as follows. First, we sample a connected Erd\H{o}s--R\'enyi graph\footnote{Using the function $\texttt{networkx.erdos\_renyi\_graph}$ at: \url{https://networkx.org/documentation/stable/reference/generated/networkx.generators.random_graphs.erdos_renyi_graph.html}} $G_1 \sim \mathcal{G}(n, p_{\mathrm{edge}})$ with $n$ nodes and edge probability $p_{\mathrm{edge}}=0.2$. If a sampled graph is disconnected, we resample until a connected graph is obtained. Let $A_1$ denote the adjacency matrix of $G_1$. Next, we generate a random hidden permutation $\pi$ over $\{1,\dots,n\}$ and construct a permuted copy of $G_1$ by $A_2^{\mathrm{clean}} = P_\pi^\top A_1 P_\pi,$ where $P_\pi$ is the permutation matrix associated with $\pi$. This produces a graph that is isomorphic to $G_1$ but with relabeled nodes. To make the alignment problem nontrivial, we then inject structural noise into the second graph by independently flipping each off-diagonal edge with probability $\eta = 0.1$. More precisely, for each unordered pair $(i,j)$ with $i<j$, the edge indicator is replaced by its complement with probability $\eta$, while symmetry and a zero diagonal are preserved. If the resulting noisy graph is disconnected, we regenerate the noise until a connected graph is obtained. The final graph is denoted by $G_2$. For each graph, we compute the all-pairs shortest-path distance matrix and normalize it by its maximum entry $C_1 \leftarrow \tfrac{C_1}{\max(C_1)}, \;
C_2 \leftarrow \tfrac{C_2}{\max(C_2)}.$ These normalized distance matrices are used as inputs to the GWOT solvers. Since $G_2$ is generated from a permuted version of $G_1$ before noise is added, the hidden permutation $\pi$ provides the ground-truth node correspondence. This enables us to evaluate the quality of the recovered alignment quantitatively.

\paragraph{Baseline methods.}
We consider 20 independently generated instances using the 20 random seeds for each problem size.
All reported statistics are averaged over these 20 runs. We compare the proposed iPG with the following GWOT solvers implemented in the open-source POT package\footnote{\url{https://pythonot.github.io/}}. \textbf{CG}: the standard conditional-gradient GW solver. \footnote{Implemented in \texttt{ot.gromov.gromov\_wasserstein}.} \textbf{PPA}: the entropic GW solver with the proximal point algorithm backend. \footnote{Implemented in \texttt{ot.gromov.entropic\_gromov\_wasserstein} with \texttt{solver="PPA"}.} \textbf{EPGD}: the entropic GW solver with projected gradient descent backend. \footnote{Implemented in \texttt{ot.gromov.entropic\_gromov\_wasserstein} with \texttt{solver="PGD"}.}  \textbf{BAPG}: the Bregman alternating projected gradient method. \footnote{Implemented \texttt{ot.gromov.BAPG\_gromov\_wasserstein}.} \textbf{iPG}: the proposed inexact projected gradient method. \footnote{Code hosted at: \url{https://github.com/liangling98/GWIPG}}

All methods are applied to the same normalized shortest-path matrices $(C_1,C_2)$ and uniform marginals $(p,q)$ with the same initialization $\Pi^0 = pq^{\top}$. For each returned coupling $\Pi$, we record the final GW loss reported by the solver. This measures the quality of the structural alignment in terms of the optimization objective. Lower values indicate better agreement between the relational geometries of the two graphs. To evaluate how well the computed coupling satisfies the transport constraints, we use the residual $\|\Pi1_n-p\|_2 + \|\Pi^\top1_n-q\|_2$. Smaller values indicate better feasibility with respect to the prescribed marginals. Since exact optimal transport solutions are often sparse, it is informative to examine the sparsity pattern of the returned coupling. We measure sparsity as the fraction of zero entries $\tfrac{\#\{(i,j)\,:\,\Pi_{ij}=0\}}{n^2}$. A higher sparsity score indicates a more concentrated transport plan. To evaluate alignment quality relative to the ground truth, we convert the soft coupling $\Pi$ into a hard permutation-like assignment $P_{\mathrm{pred}}$ using the Hungarian algorithm \footnote{Implemented in $\texttt{scipy.optimize.linear\_sum\_assignment}$ at: \url{https://docs.scipy.org/doc/scipy/reference/generated/scipy.optimize.linear_sum_assignment.html}}, and then compare it with the true permutation matrix $P_\pi$ via $\tfrac{1}{n}\langle P_\pi, P_{\mathrm{pred}} \rangle$.

\paragraph{Solver hyperparameters.}
For the CG solver, we set the maximum number of outer iterations to $\texttt{max\_iter}=5000$. For the entropic solvers PPA, EPGD, and BAPG, we test four entropic regularization levels $\epsilon \in \{10^{-3},\,10^{-2},\,10^{-1},\,10^{0}\}$. Each solver is run with a maximum of $\texttt{max\_iter} = 5000$ iterations. This allows us to examine the effect of entropy strength on objective value, feasibility, sparsity, and runtime. For the proposed iPG method, we terminate the algorithm when the maximum iterations $\texttt{max\_iter} = 5000$ or the relative successive change of iterates is less than $\texttt{tol} = 10^{-9}$. The projection subproblems are solved by the accelerated gradient method \cite{beck2009fast} with $\varepsilon_k := (1+k)^{-3}$, $k\geq 0$. For iPG, we initialize the inverse step size as $\gamma_0 = 4\|C_1\|_2\|C_2\|_2/(10n)$ and apply heuristic secant updates every 100 outer iterations whenever the most recently monitored successive-iterate change is below $10^{-4}$. We also cap the inner iterations (Appendix \ref{appendix-projection-cost}) to improve the efficiency, although the convergence guarantees require \eqref{eq-inexact-ipg-subpro} at every outer iteration. Runs that terminate with an exception are marked as failures and excluded from the statistical averages, while the number of successful and failed runs is tracked separately.

\paragraph{Computational results.} 
The computational results are summarized in Figure~\ref{fig:metrics}; detailed numerical records can be found in Appendix~\ref{appendix-detailed-results}. Overall, the proposed iPG method delivers a favorable balance among GW objective value, sparsity, feasibility, alignment accuracy, and runtime across the tested instances. In particular, iPG attains losses comparable to those of competitive baselines at $n\in\{100,200\}$ and the smallest losses at $n\in\{300,400,500\}$, while maintaining extremely small feasibility residuals. At the same time, it returns couplings with sparsity nearly equal to one, which is much closer to the expected structure of an exact transport solution than the couplings produced by the entropic baselines. This structural advantage is important in graph alignment, since a sparse and nearly feasible coupling is substantially easier to interpret and to round into a high-quality discrete correspondence. iPG also achieves the highest alignment accuracy at the larger tested sizes.

The comparisons also clarify the strengths and weaknesses of the competing approaches. The classical CG method always produces very sparse and exactly feasible couplings, but its loss and matching accuracy are generally inferior to those of other methods, suggesting that feasibility and sparsity alone do not guarantee a good alignment if the method converges to a poorer stationary point. EPGD is computationally cheap and achieves very small feasibility residuals for moderate and large values of the entropic parameter, but its iterates remain essentially dense across all problem sizes, which is undesirable for recovering a discrete matching. More broadly, the entropic PG-type approach does not preserve the sparse structure that one expects from the unregularized GWOT solution, and feasibility can also become problematic when the regularization is small. PPA is more competitive in terms of loss and accuracy, and for some values of $\varepsilon$, it performs similarly to iPG; however, it is clearly sensitive to the choice of $\varepsilon$, returns noticeably denser couplings than iPG, and becomes substantially slower as $\varepsilon$ increases. BAPG is the least robust among the tested methods: it fails when $\varepsilon$ is small, and even when it succeeds, its feasibility and alignment accuracy are often unsatisfactory unless the regularization is tuned carefully.

Taken together, these results support the central message of the work. A non-entropic projected-gradient framework with practical inexact projections can preserve the expected structural features of unregularized GWOT solutions, in particular near-feasibility and high sparsity, while remaining computationally efficient. In contrast, entropic methods often introduce a smoothing bias that yields denser couplings and may weaken downstream matching performance, whereas exact-feasibility methods such as CG can converge to less favorable stationary points in terms of objective value and alignment accuracy. Overall, the proposed iPG method offers a robust, scalable, and structurally faithful approach to graph alignment via GWOT.

\begin{figure}[!]
    \centering
    \begin{minipage}[t]{0.48\textwidth}
        \centering
        \includegraphics[width=\linewidth]{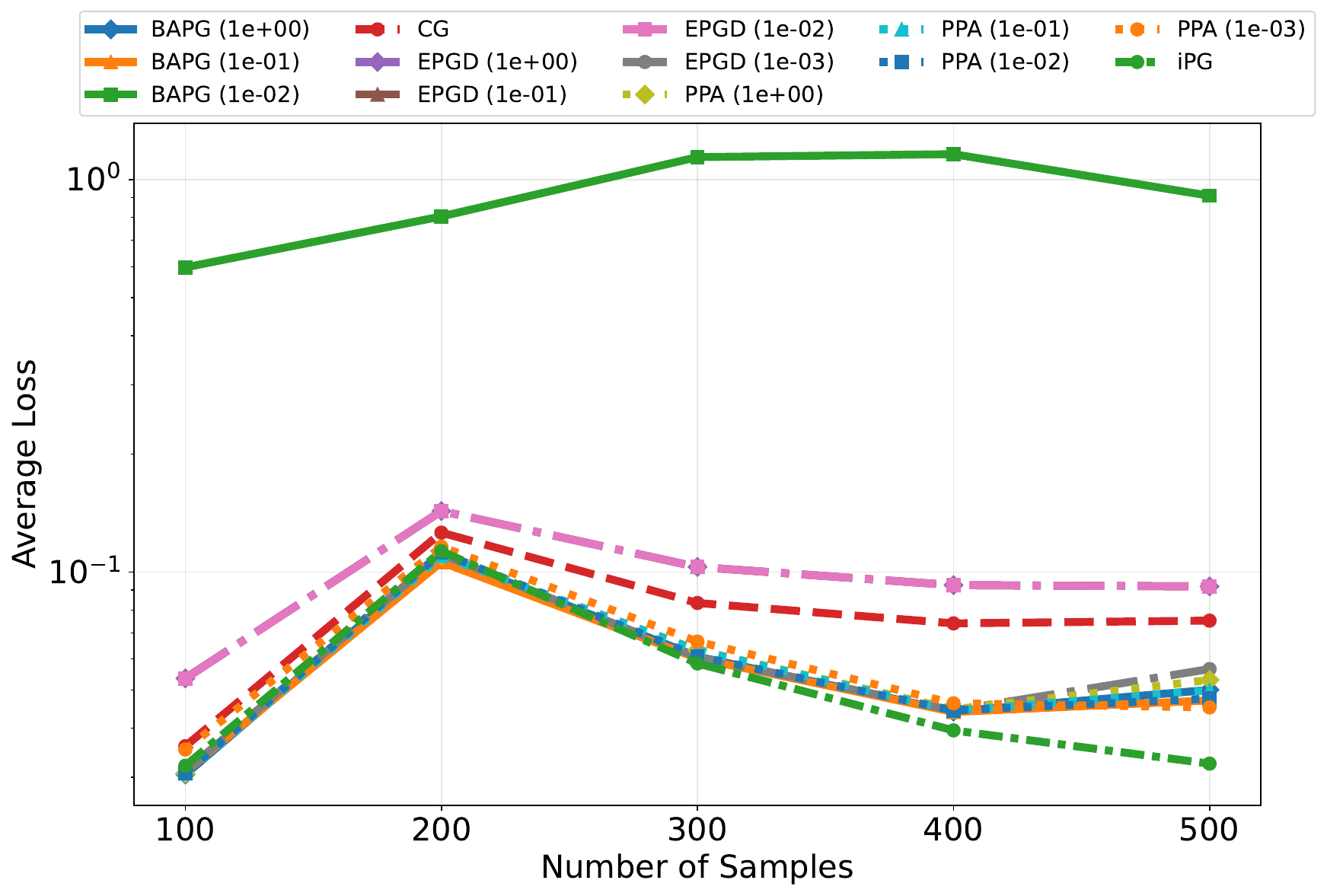}

        \small (a) Loss
    \end{minipage}\hfill
    \begin{minipage}[t]{0.48\textwidth}
        \centering
        \includegraphics[width=\linewidth]{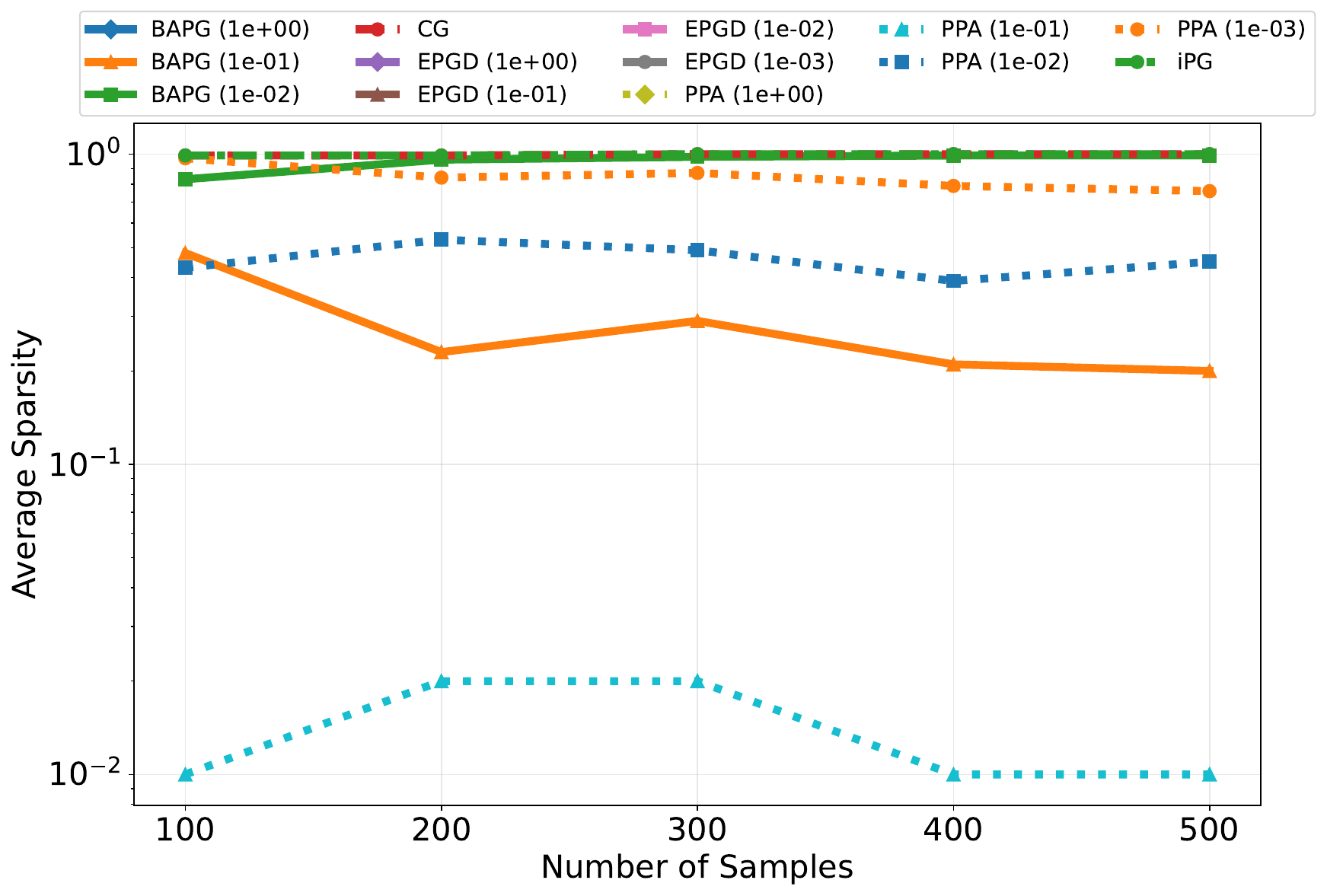}

        \small (b) Sparsity
    \end{minipage}

    \medskip

    \begin{minipage}[t]{0.48\textwidth}
        \centering
        \includegraphics[width=\linewidth]{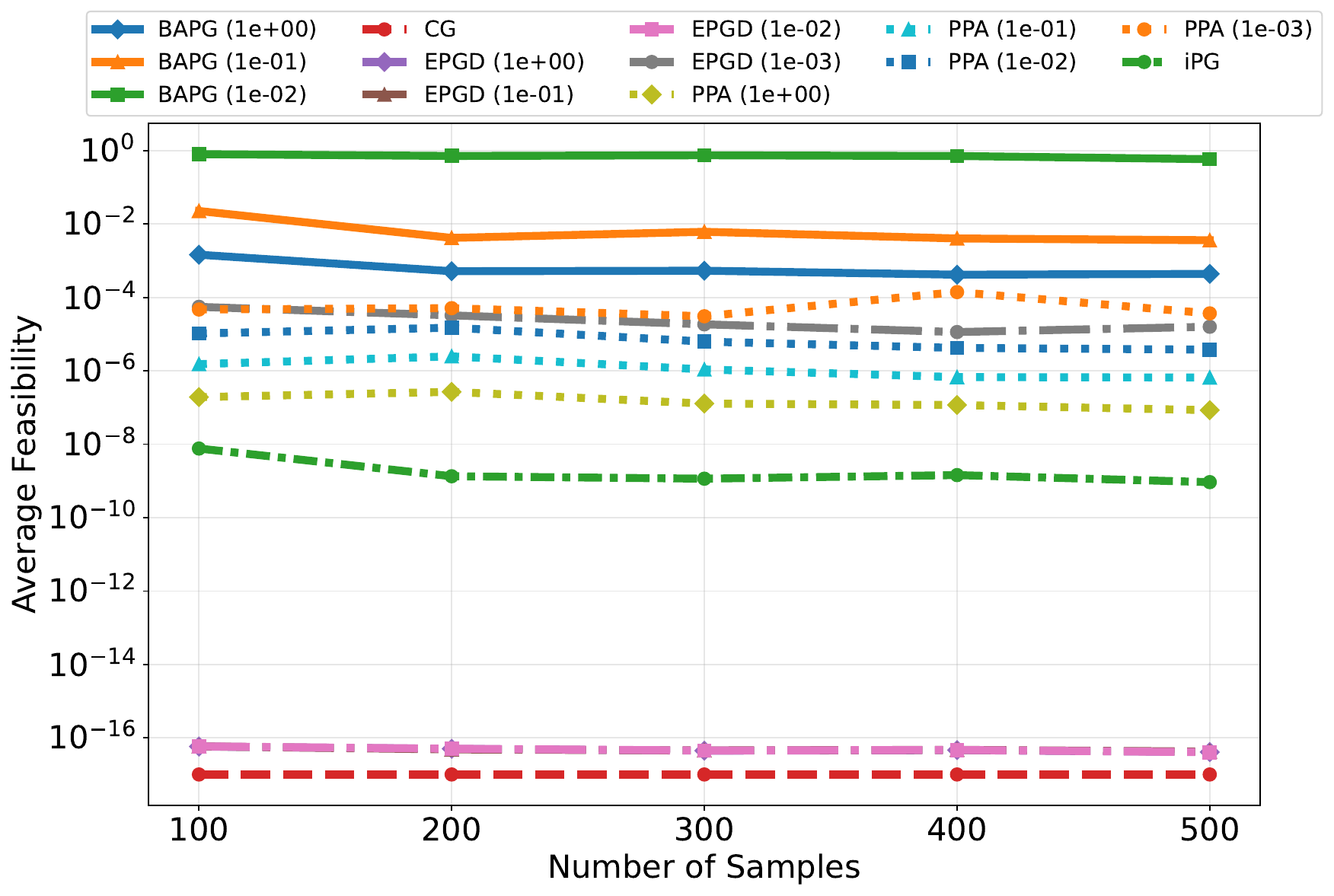}

        \small (c) Feasibility
    \end{minipage}\hfill
    \begin{minipage}[t]{0.48\textwidth}
        \centering
        \includegraphics[width=\linewidth]{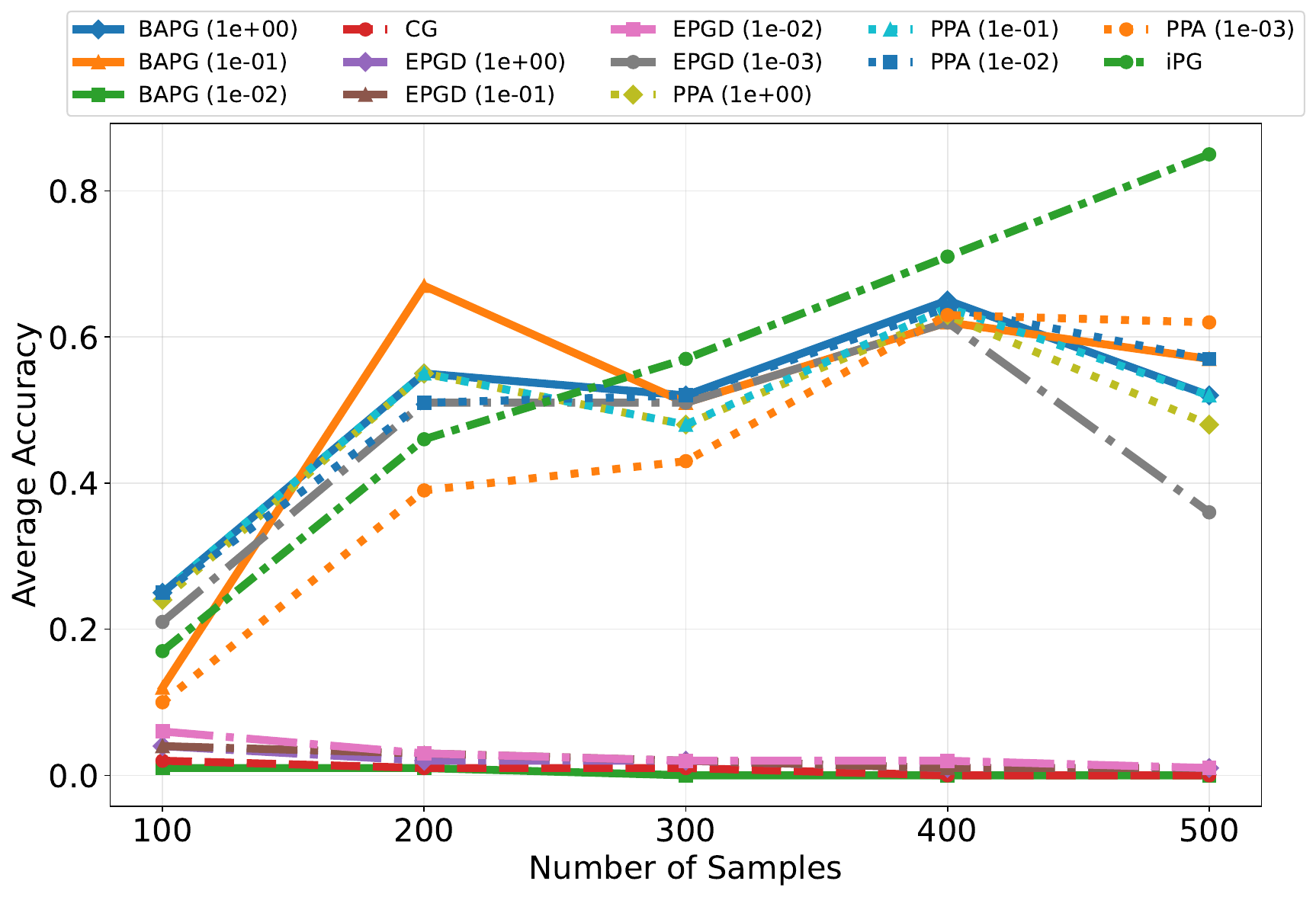}

        \small (d) Accuracy
    \end{minipage}

    \medskip

    \begin{minipage}[t]{0.48\textwidth}
        \centering
        \includegraphics[width=\linewidth]{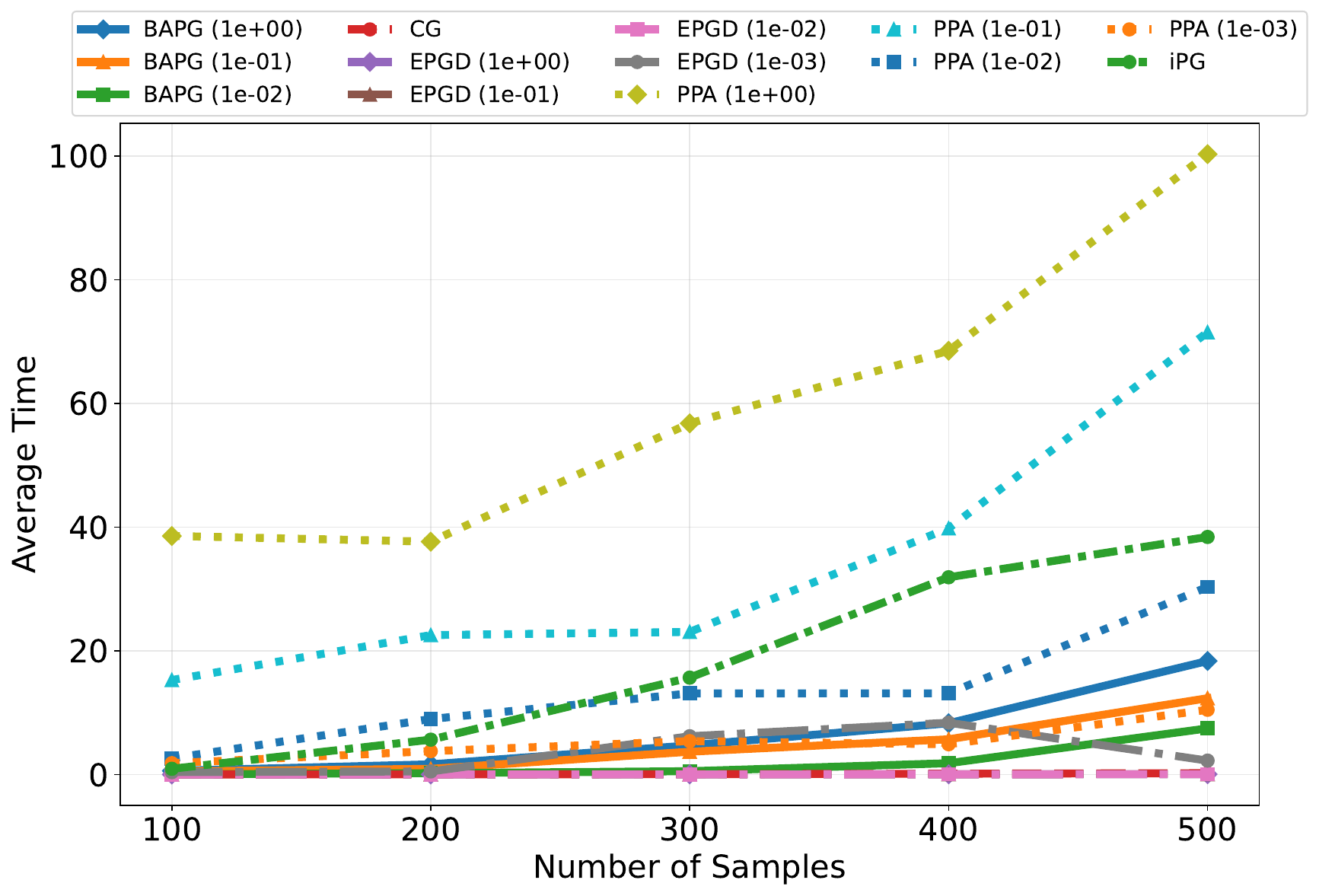}

        \small (e) Time
    \end{minipage}

    \caption{Computation results with $n\in \{100,200,300,400,500\}$.}
    \label{fig:metrics}
    \vspace{-2em}
\end{figure}

\subsection{Paired single-cell alignment}

We use the 1,047 paired cells in the SNARE-seq benchmark distributed with SCOT~\cite{ChenLakeZhang2019SNARE,DemetciEtAl2020SCOT}, together with subsets of 250 and 500 cells. Following SCOT, we construct a 50-nearest-neighbor graph in each modality and normalize its shortest-path distances. The RNA ordering is permuted to hide the known pairing. We report five-run means of GW loss, feasibility, sparsity, runtime, the fraction of samples (FOS) closer than the true match, and 5-nearest-neighbor label-transfer accuracy (LTA) from chromatin to RNA. Appendix~\ref{appendix-snare} gives the protocol, full parameter sweep, and variability.

\begin{table}[htbp]
\centering\small
\setlength{\tabcolsep}{8pt}
\caption{Selected results on the full SNARE-seq dataset ($n=1047$) using five-run means.}
\label{tab:snare-main}
\begin{tabular}{lrrrrrrr}
\toprule
Method & $\epsilon$ & Loss & Feasibility & Sparse (\%) & FOS & LTA & Time (s) \\
\midrule
CG & -- & $0.03743$ & $0$ & $99.90$ & $0.269$ & $0.668$ & $5.78$ \\
PPA & $5\!\times\!10^{-4}$ & $0.03564$ & $1.49\!\times\!10^{-3}$ & $96.96$ & $0.155$ & $0.979$ & $156.47$ \\
EPGD & $5\!\times\!10^{-4}$ & $0.03511$ & $1.67\!\times\!10^{-4}$ & $2.81$ & $0.150$ & $0.985$ & $25.98$ \\
EPGD & $10^{-2}$ & $0.04427$ & $9.39\!\times\!10^{-10}$ & $0.00$ & $0.176$ & $0.847$ & $4.52$ \\
BAPG & $10^{-1}$ & $0.03247$ & $3.01\!\times\!10^{-2}$ & $0.00$ & $0.624$ & $0.144$ & $1.41$ \\
\midrule
iPG & -- & $0.03504$ & $4.81\!\times\!10^{-6}$ & $99.19$ & $0.150$ & $0.985$ & $21.28$ \\
\bottomrule
\end{tabular}
\end{table}

On the full dataset, iPG and EPGD with $\epsilon=5\times10^{-4}$ both attain FOS $0.150$ and LTA $0.985$. iPG gives a smaller feasibility residual ($4.81\times10^{-6}$ versus $1.67\times10^{-4}$), 99.19\% exact zeros versus 2.81\%, and mean runtime 21.28 versus 25.98 seconds. Its sparsity corresponds to about 8.5 candidate matches per cell. CG is faster and sparser but less accurate, while stronger entropic regularization can reduce runtime at the expense of alignment quality. These results support a tradeoff among alignment, feasibility, sparsity, and runtime rather than uniform superiority. Sparse couplings give shorter candidate lists and reduce downstream storage and feature-transfer costs.

\subsection{Additional diagnostics}

To better understand the practical behavior of iPG, we examine its computational bottlenecks, sensitivity to inner accuracy, and agreement with analytical reference values. Projection accounts for 84.7\%--95.7\% of internal solver time (Table~\ref{tab:projection-cost}). Projection is thus the dominant cost at these sizes. On a fixed 250-cell instance, the inner-accuracy ablation shows that tighter tolerances can improve feasibility with little change in GW loss, at the expense of additional projection work. Sphere benchmarks give mean relative errors of 0.60\%--8.32\% against analytical continuous $(4,2)$-GW distances (Table~\ref{tab:sphere-reference}); these combine sampling and optimization effects. Appendix~\ref{appendix-projection-cost}--\ref{appendix-spheres} reports these results and the inner-accuracy ablation.

\section{Conclusions}\label{sec-conclusion}

In this paper, we developed a provably convergent inexact projected gradient framework for solving the GWOT problem, with a particular emphasis on \emph{practical implementability with convergence guarantees}. Our main contribution is a truly verifiable inexact condition for the approximate projection step onto the transport polytope, which avoids unknown quantities, admits a simple feasibility-residual-based stopping rule, and helps bridge the gap between scalable implementations and rigorous convergence theory. Under this implementable condition, we established subsequential convergence to stationary points and, under an additional mild tolerance-decay condition, convergence of the whole sequence. Numerical experiments on graph alignment further show that the proposed iPG method achieves a favorable balance among objective quality, sparsity, feasibility, alignment accuracy, and computational efficiency, compared with several representative GWOT solvers.

Looking ahead, several limitations of the current work merit further investigation. On the application side, it will be important to deploy the proposed framework in more real-world GWOT tasks, such as network alignment, shape analysis, biological data integration, and cross-domain representation matching. On the computational side, an important next step is to develop high-performance implementations that better exploit modern hardware architectures, including multi-core CPUs, GPUs, and distributed systems, so as to further enhance scalability for large-scale problems. It is also of strong interest to design stochastic and mini-batch variants of the proposed method possibly with variance reduction techniques, which may provide additional efficiency gains in massive data settings and open the door to online or streaming GWOT computations.

\clearpage
\begin{ack}
The authors thank the program chairs and reviewers for their insightful comments and constructive suggestions, which have helped improve the paper. Ling Liang acknowledges support from the Department of Mathematics at the University of Tennessee, Knoxville, through a summer research fellowship. The research of Lei Yang is supported in part by the National Natural Science Foundation of China under grant 12301411, and the Natural Science Foundation of Guangdong under grant 2026A1515010350.
\end{ack}

\bibliographystyle{plain}
\bibliography{Ref_GW}

\newpage
\appendix

\section{Proof of Theorem \ref{thm-subsequential-convergence}}\label{appendix-proof-sub-conv}

In this section, we provide a detailed proof for the subsequential convergence of Algorithm \ref{algo-iPG}.

\begin{lemma}[Explicit error bound for $\mathcal{U}(a,b)$ when $\Pi\ge 0$]\label{lem:explicit-hoffman-Uab}
Let $a\in\mathbb{R}^n_+$ and $b\in\mathbb{R}^m_+$ satisfy $1^\top_n a = 1^\top_m b$. For any $\Pi\in\mathbb{R}^{n\times m}_+$, there exists $\widetilde{\Pi}\in\mathcal{U}(a,b)$ such that
\begin{equation*}
\|\Pi - \widetilde{\Pi}\|_F
\leq \|\Pi-\widetilde{\Pi}\|_{1}
\leq 2\big(\|a-\Pi1_m\|_1+\|b-\Pi^\top1_n\|_1\big)
\leq 2\sqrt{n+m}\,\left\|\mathcal{A}\Pi-r\right\|_2.
\end{equation*}
\end{lemma}
\begin{proof}
The proof is straightforward with the second inequality following from \cite[Lemma 7]{altschuler2017near}.
\end{proof}

Note that the point $\widetilde{\Pi}$ constructed in Lemma~\ref{lem:explicit-hoffman-Uab} need not coincide with the Euclidean projection of $\Pi$ onto the feasible set $\mathcal{U}(a,b)$.

\begin{proof}[Proof of Theorem \ref{thm-subsequential-convergence}]
The existence of the sequence $\{\widetilde{\Pi}^k\}$ follows from
Lemma~\ref{lem:explicit-hoffman-Uab}.

\textbf{Statement 1.}
The boundedness of $\{\widetilde{\Pi}^k\}$ follows from the boundedness
of the feasible set $\mathcal{U}(a,b)$. In particular,
$\|\widetilde{\Pi}^k\|_F\leq \|\widetilde{\Pi}^k\|_{1}=1$ for all
$k\geq0$. Moreover, by the construction of $\{\widetilde{\Pi}^k\}$,
the definition of $\varepsilon_{-1}$, and the inexact condition
\eqref{eq-inexact-ipg-subpro}, we have that
\begin{equation*}
\|\Pi^k\|_F
\leq \|\Pi^k-\widetilde{\Pi}^k\|_F+\|\widetilde{\Pi}^k\|_F
\leq 2\sqrt{m+n}\,\varepsilon_{k-1} + 1
\leq 2E\sqrt{m+n} + 1,
\quad k\geq0,
\end{equation*}
where $\varepsilon_{-1}:=\|\mathcal{A}\Pi^0-r\|_2$ and
$E:=\sum_{k=-1}^{\infty}\varepsilon_k<\infty$. Thus, $\{\Pi^k\}$ is also bounded.

\textbf{Statement 2.} Define
\begin{equation*}
\Gamma^{k+1} := \gamma \left( \Pi^{k+1}  - \left(\mathcal{A}^*y^{k+1} + \Pi^k - \frac{1}{\gamma}\nabla f(\Pi^k)\right)\right), \quad k\geq 0.
\end{equation*}
It is easy to verify that $\Gamma^{k+1}\geq0$, $\langle \Pi^{k+1},\Gamma^{k+1}\rangle=0$ (by the construction of $\Pi^{k+1}$ in \eqref{defHk}), and
\begin{equation*}
\nabla f(\Pi^k) + \gamma (\Pi^{k+1} - \Pi^k) - \gamma\mathcal{A}^*y^{k+1} - \Gamma^{k+1} = 0.
\end{equation*}
Moreover, by Statement 1, $\|\mathcal{A}\|_2\leq\sqrt{m+n}$, and $f(\Pi):= -\langle C_{\mathcal{X}}\Pi C_{\mathcal{Y}},\Pi\rangle$ with $L_f<\gamma$, we have
\begin{equation*}
\begin{aligned}
\|\Gamma^{k+1}\|_F
\leq &\; \gamma\|\mathcal{A}^*y^{k+1}\|_F + \gamma\|\Pi^{k+1} - \Pi^k\|_F
+ \|\nabla f(\Pi^k)\|_F \\
\leq &\; \gamma \|\mathcal{A}\|_2 \|y^{k+1}\|_2 + \gamma\|\Pi^{k+1}\|_F
+ (\gamma + L_f)\|\Pi^k\|_F  \\
\leq &\; \gamma \sqrt{m+n} \|y^{k+1}\|_2 + 3\gamma\big(2E\sqrt{m+n} + 1\big).
\end{aligned}
\end{equation*}
Then, for any $\Pi\in \mathcal{U}(a,b)$, we see that
\begin{equation*}
\begin{aligned}
&-\langle\nabla f(\Pi^k) + \gamma (\Pi^{k+1} - \Pi^k), \,\Pi - \widetilde{\Pi}^{k+1}\rangle
= \langle -\gamma\mathcal{A}^*y^{k+1} - \Gamma^{k+1}, \,\Pi - \widetilde{\Pi}^{k+1}\rangle  \\
= &\; \gamma\langle y^{k+1}, \,\mathcal{A}\widetilde{\Pi}^{k+1} \!-\! \mathcal{A}\Pi \rangle
+ \langle \Gamma^{k+1}, \,\widetilde{\Pi}^{k+1} \!-\! \Pi\rangle
= \langle \Gamma^{k+1}, \,\widetilde{\Pi}^{k+1} \!-\! \Pi^{k+1}\rangle + \langle \Gamma^{k+1},\Pi^{k+1}\rangle  -  \langle  \Gamma^{k+1}, \Pi\rangle \\
\leq &\; \|\Gamma^{k+1}\|_F\|\widetilde{\Pi}^{k+1}-\Pi^{k+1}\|_F
\leq 2\sqrt{m+n}\,\|\Gamma^{k+1}\|_F\|\mathcal{A}\Pi^{k+1}-r\|_2  \\
\leq &\; 2\sqrt{m+n}\Big(\gamma \sqrt{m+n} \|y^{k+1}\|_2 + 3\gamma\big(2E\sqrt{m+n} + 1\big)\Big)\frac{\varepsilon_k}{1 + \|y^{k+1}\|_2} \\
\leq &\; (6E+4)2\gamma(m+n)\varepsilon_k =: \widehat{\varepsilon}_k,
\end{aligned}
\end{equation*}
where the third equality follows from $\mathcal{A}\widetilde{\Pi}^{k+1}=\mathcal{A}\Pi=r$, the first inequality follows from $\langle \Gamma^{k+1},\Pi^{k+1}\rangle=0$ and $\langle  \Gamma^{k+1}, \Pi\rangle\geq0$, the second inequality follows from the construction of $\{\widetilde{\Pi}^k\}$, and the third inequality follows from the inexact condition \eqref{eq-inexact-ipg-subpro}. This implies that $D^{k+1}:= -\nabla f(\Pi^k) - \gamma(\Pi^{k+1} - \Pi^k)$ is an $\widehat{\varepsilon}_k$-subgradient of $\delta_{\mathcal{U}(a,b)}$ at the point $\widetilde{\Pi}^{k+1}$. By taking $\Pi = \widetilde{\Pi}^k\in \mathcal{U}(a,b)$ in the above inequality, we obtain that
\begin{equation*}
\begin{aligned}
\widehat{\varepsilon}_k \geq &\;  -\langle\nabla f(\Pi^k) + \gamma (\Pi^{k+1} - \Pi^k),  \widetilde{\Pi}^k - \widetilde{\Pi}^{k+1}\rangle  \\
= &\; -\langle \nabla f(\Pi^k), \widetilde{\Pi}^k - \widetilde{\Pi}^{k+1}\rangle - \gamma \langle \Pi^{k+1} - \Pi^k, \widetilde{\Pi}^k - \widetilde{\Pi}^{k+1}\rangle \\
= &\; -\langle \nabla f(\Pi^k), \widetilde{\Pi}^k - \widetilde{\Pi}^{k+1}\rangle \\
&\; - \frac{\gamma}{2} \left( \|\widetilde{\Pi}^k - \Pi^k\|_F^2-\|\widetilde{\Pi}^k - \Pi^{k+1} \|_F^2- \|\widetilde{\Pi}^{k+1} - \Pi^k\|_F^2 + \|\widetilde{\Pi}^{k+1} - \Pi^{k+1}\|_F^2\right)  \\
\geq &\; \langle \nabla f(\Pi^k), \widetilde{\Pi}^{k+1}-\widetilde{\Pi}^k\rangle
+ \frac{\gamma}{2} \|\widetilde{\Pi}^k - \Pi^{k+1} \|_F^2 + \frac{\gamma}{2}   \|\widetilde{\Pi}^{k+1} - \Pi^k\|_F^2  - 2\gamma  (m+n)(\varepsilon_{k-1}^2  + \varepsilon_{k}^2).
\end{aligned}
\end{equation*}
Since $f$ is $L_f$-smooth, we have that
\begin{equation*}
\begin{aligned}
f(\widetilde{\Pi}^{k+1}) \leq &\; f(\Pi^k) + \langle \nabla f(\Pi^k), \widetilde{\Pi}^{k+1} - \Pi^k\rangle + \frac{L_f}{2}\|\widetilde{\Pi}^{k+1} - \Pi^k\|_F^2, \\
f(\widetilde{\Pi}^{k}) \geq &\; f(\Pi^k) + \langle \nabla f(\Pi^k), \widetilde{\Pi}^{k} - \Pi^k\rangle  - \frac{L_f}{2}\|\widetilde{\Pi}^{k} - \Pi^k\|_F^2,
\end{aligned}
\end{equation*}
which leads to
\begin{equation*}
\begin{aligned}
f(\widetilde{\Pi}^{k+1}) \leq  &\;  f(\widetilde{\Pi}^{k}) + \langle \nabla f(\Pi^k), \widetilde{\Pi}^{k+1} - \widetilde{\Pi}^{k}\rangle  + \frac{L_f}{2}\|\widetilde{\Pi}^{k+1} - \Pi^k\|_F^2 + 2L_f (m+n)\varepsilon_{k-1}^2 \\
\leq &\;  f(\widetilde{\Pi}^{k}) + 2\gamma  (m+n)(\varepsilon_{k-1}^2  + \varepsilon_{k}^2) + \widehat{\varepsilon}_k +  2\gamma  (m+n)\varepsilon_{k-1}^2  \\
&\; - \frac{\gamma}{2} \|\widetilde{\Pi}^k - \Pi^{k+1} \|_F^2  - \frac{\gamma  - L_f}{2}  \|\widetilde{\Pi}^{k+1} - \Pi^k\|_F^2.
\end{aligned}
\end{equation*}
This then implies \eqref{eq-sufficient-descent} by setting $\xi_k:=4\gamma(m+n)(\varepsilon_{k-1}^2 + \varepsilon_{k}^2) + \widehat{\varepsilon}_k$.

\textbf{Statement 3.} It follows from \eqref{eq-sufficient-descent} that
\begin{equation*}
f(\widetilde{\Pi}^{k+1}) \leq f(\widetilde{\Pi}^{k}) + \xi_k.
\end{equation*}
Since $\sum_{k=0}^\infty \varepsilon_k < \infty$, one can verify that $\sum_{k=0}^\infty \xi_k < \infty$. Moreover, since $\mathcal{U}(a,b)$ is compact and $f$ is continuous, $f$ is bounded below on $\mathcal{U}(a,b)$. Then, we can conclude that $\{f(\widetilde{\Pi}^k)\}$ is convergent, i.e., the limit $\lim\limits_{k\to\infty}f(\widetilde{\Pi}^k)$ exists; see, e.g., \cite[Section 2.2]{polyak1987introduction}. From \eqref{eq-sufficient-descent} and the continuity of $f$, we further see that
\begin{equation*}
0 \leq \frac{\gamma}{2} \|\widetilde{\Pi}^k - \Pi^{k+1} \|_F^2 + \frac{\gamma  - L_f}{2}  \|\widetilde{\Pi}^{k+1} - \Pi^k\|_F^2
\leq  f(\widetilde{\Pi}^{k}) - f(\widetilde{\Pi}^{k+1}) + \xi_k \to 0,
\end{equation*}
as $k\to \infty$. This, together with $\gamma > L_f$, proves that
\begin{equation*}
\lim_{k\to\infty}\; \|\widetilde{\Pi}^{k+1} - \Pi^k\|_F = \lim_{k\to\infty}\; \|\widetilde{\Pi}^{k} - \Pi^{k+1}\|_F = 0.
\end{equation*}
Finally, since $\|\widetilde{\Pi}^{k} - \Pi^{k}\|_F\leq 2\sqrt{m+n}\,\|\mathcal{A}\Pi^k-r\|_2\leq2\sqrt{m+n}\,\varepsilon_{k-1}\to0$, we then obtain that any accumulation point of $\{\Pi^k\}$ is an accumulation point of $\{\widetilde{\Pi}^{k}\}$, and vice versa.

\textbf{Statement 4.} We only need to show that any accumulation point $\Pi^*\in \mathcal{U}(a,b)$ of $\{\widetilde{\Pi}^k\}$ is a stationary point of the GWOT problem \eqref{eq:gw}. Since $\{\widetilde{\Pi}^k\}$ is bounded (by Statement 1), it has at least one accumulation point. Suppose that $\Pi^*$ is an accumulation point and $\{\widetilde{\Pi}^{k_i}\}$ is a convergent subsequence such that $\lim\limits_{i\to \infty}\widetilde{\Pi}^{k_i} = \Pi^*$. Since $\|\Pi^{k_i} - \widetilde{\Pi}^{k_i}\|_F\to 0$, we see that $\lim\limits_{i\to \infty}\Pi^{k_i} = \Pi^*$ and from Statement 3 that
\begin{equation}\label{succlim}
\begin{aligned}
&\; \|\Pi^{k_i+1} - \Pi^{k_i}\|_F
\leq \|\Pi^{k_i+1} - \widetilde{\Pi}^{k_i}\|_F + \|\widetilde{\Pi}^{k_i} - \Pi^{k_i}\|_F \to 0,  \\
&\; \|\widetilde{\Pi}^{k_i+1} - \widetilde{\Pi}^{k_i}\|_F \leq \|\widetilde{\Pi}^{k_i+1} - \Pi^{k_i+1}\|_F + \|\Pi^{k_i+1} - \widetilde{\Pi}^{k_i}\|_F \to 0,
\end{aligned}
\end{equation}
as $i\to \infty$. Recall that $D^{k+1}=-\nabla f(\Pi^k) - \gamma(\Pi^{k+1} - \Pi^k)$ is an $\widehat{\varepsilon}_k$-subdifferential of $\delta_{\mathcal{U}(a,b)}$ at $\widetilde{\Pi}^{k+1}$, i.e.,
\begin{equation*}
-\gamma (\Pi^{k+1} - \Pi^k) \in \partial_{\widehat{\varepsilon}_k}\delta_{\mathcal{U}(a,b)}(\widetilde{\Pi}^{k+1}) + \nabla f(\Pi^k).
\end{equation*}
Then, passing to the limit in above relation and invoking \eqref{succlim}, the Lipschitz continuity of $\nabla f$, and the outer semicontinuity of $\partial_{\widehat{\varepsilon}_k}\delta_{\mathcal{U}(a,b)}$ (see, for example, \cite[Proposition 4.1.1]{hl1993convex}) with $\widehat{\varepsilon}_k\to0$ (due to $\varepsilon_{k}\to0$), we see that
\begin{equation*}
0 \in \partial \delta_{\mathcal{U}(a,b)}(\Pi^{*}) + \nabla f(\Pi^*),
\end{equation*}
which implies that $\Pi^*$ is a stationary point of the GWOT problem \eqref{eq:gw}.

Finally, pick two arbitrary accumulation points $\Pi^*$ and $\Pi^\infty$ associated with the subsequences $\widetilde{\Pi}^{k_i}$ and $\widetilde{\Pi}^{k_j}$, respectively. By the continuity of the function $f$ and the fact that $\{f(\widetilde{\Pi}^k)\}$ is convergent, we see that $f(\Pi^*) = \lim\limits_{i\to\infty}f(\widetilde{\Pi}^{k_i})= \lim\limits_{k\to\infty}f(\widetilde{\Pi}^k) = \lim\limits_{j\to\infty}f(\widetilde{\Pi}^{k_j}) = f(\Pi^\infty)$. Hence, we conclude that $f$ is constant on the set of accumulation points and complete the proof.
\end{proof}

\section{Proof of Proposition \ref{prop:projection-error-bound}}

\begin{lemma}[\textbf{Uniform error bound for projection}]\label{lem:uniform-hoffman-projection}
For each $Z\in\mathbb{R}^{n\times m}$, consider the projection problem
\begin{equation*}
\min_{\Pi\in\mathcal{U}(a,b)}~\frac{1}{2}\|\Pi-Z\|_F^2 .
\end{equation*}
Let $\Pi^*(Z)$ denote its unique solution. Then, there exists a constant $\kappa>0$, independent of all $Z$, such that for every $y\in\mathbb{R}^{n+m}$ and
\begin{equation*}
\Pi(y,Z):=\mathrm{Proj}_{\mathbb{R}_+^{n\times m}}\left(\mathcal{A}^*y+Z\right),
\end{equation*}
one has
\begin{equation*}
\|\Pi(y,Z)-\Pi^*(Z)\|_F\leq\kappa\,\|\mathcal{A}\Pi(y,Z)-r\|_2.
\end{equation*}
\end{lemma}
\begin{proof}
For simplicity, we vectorize matrices and write the transport constraints in the form
\begin{equation*}
A\pi=r,\quad \pi\geq0,
\end{equation*}
where $\pi:=\mathrm{vec}(\Pi)$ and $A$ is the matrix representation of $\mathcal{A}$. Similarly, let $z:=\mathrm{vec}(Z)$. Then, the projection problem is equivalent to
\begin{equation*}
\min_{A\pi=r,\,\pi\geq0}~\frac{1}{2}\|\pi-z\|_2^2 .
\end{equation*}
Let $\mathcal{Q}:=A\mathbb{R}^{nm}_+$ be the set of feasible right-hand sides. For any $q\in\mathcal{Q}$, we further define
\begin{equation}\label{prob-para-q}
\pi^*(z,q):=\operatorname*{argmin}_{A\pi=q,\,\pi\geq0}
~\frac{1}{2}\|\pi-z\|_2^2.
\end{equation}
This solution is well defined and unique because the feasible set is nonempty, closed, and convex, while the objective is strongly convex. In particular, $\mathrm{vec}(\Pi^*(Z))=\pi^*(z,r)$.

In the following, we first show that there exists a constant $\kappa>0$, depending only on $A$, such that
\begin{equation}\label{eq:parametric-projection-bound}
\|\pi^*(z,q_1)-\pi^*(z,q_2)\|_2\leq\kappa\|q_1-q_2\|_2
\quad\text{for all }z\in\mathbb{R}^{nm}\text{ and }q_1,q_2\in\mathcal{Q}.
\end{equation}

For fixed $(z,q)\in\mathbb{R}^{nm}\times\mathcal{Q}$, the KKT conditions for the problem in \eqref{prob-para-q} are
\begin{equation*}
A\pi=q,\qquad
\pi-z-A^\top y-s=0,\qquad
0\leq\pi\perp s\geq0.
\end{equation*}
These conditions are necessary and sufficient because the objective is convex and differentiable and the feasible set is a nonempty polyhedron.

Let $\mathcal{K}(z,q)$ denote the set of all primal--dual triples $(\pi,y,s)$ satisfying the above KKT conditions. For each index set $\mathcal{I}\subseteq\{1,\ldots,nm\}$, let $\mathcal{J}:=\mathcal{I}^c$ and define the polyhedron
\begin{equation*}
\mathcal{G}_{\mathcal{I}}:=\left\{(z,q,\pi,y,s):
A\pi=q,
~~\pi-z-A^\top y-s=0,
~~\pi_{\mathcal{I}}\geq0,
~~s_{\mathcal{I}}=0,
~~\pi_{\mathcal{J}}=0,
~~s_{\mathcal{J}}\geq0\right\}.
\end{equation*}
The set $\mathcal{G}_{\mathcal{I}}$ encodes one complementarity regime: it imposes $\pi_{\mathcal{I}}\geq0$ and $s_{\mathcal{I}}=0$ on $\mathcal{I}$, and $\pi_{\mathcal{J}}=0$ and $s_{\mathcal{J}}\geq0$ on $\mathcal{J}$. Then, for any fixed $(z,q)\in\mathbb{R}^{nm}\times\mathcal{Q}$ and any $(\pi,y,s)$, we have that
\begin{equation*}
(\pi,y,s)\in\mathcal{K}(z,q)
\quad\Longleftrightarrow\quad
\text{there exists }\mathcal{I}\subseteq\{1,\ldots,nm\}
\text{ such that }(z,q,\pi,y,s)\in\mathcal{G}_{\mathcal{I}}.
\end{equation*}
Indeed, for the forward implication, choose $\mathcal{I}:=\{i:\pi_i>0\}$ and let $\mathcal{J}:=\mathcal{I}^c$. Complementarity gives $s_{\mathcal{I}}=0$, while the definition of $\mathcal{J}$ gives $\pi_{\mathcal{J}}=0$; the remaining inequalities follow from primal and dual feasibility. Conversely, membership in $\mathcal{G}_{\mathcal{I}}$ gives $\pi\geq0$, $s\geq0$, and $\pi_i s_i=0$ for every $i$, together with the two KKT equalities. The representing index set need not be unique when $\pi_i=s_i=0$ for some $i$, but uniqueness is not required for this equivalence.

Define the coordinate projection onto the $(z,q,\pi)$-space by $\operatorname{pr}_{z,q,\pi}(z,q,\pi,y,s):=(z,q,\pi)$. Thus, for any set $\mathcal{S}$ of extended tuples, its image under this projection is
\begin{equation*}
\operatorname{pr}_{z,q,\pi}(\mathcal{S})
:=
\left\{(z,q,\pi):
\text{there exists }(y,s)\text{ such that }
(z,q,\pi,y,s)\in\mathcal{S}\right\}.
\end{equation*}
By the KKT characterization above, $\mathcal{K}(z,q)$ is nonempty, and all its triples have the \textit{same} primal component $\pi^*(z,q)$. Consequently, for any $(z,q)\in\mathbb{R}^{nm}\times\mathcal{Q}$ and $\pi\in\mathbb{R}^{nm}$,
\begin{equation*}
\begin{aligned}
(z,q,\pi)\in\operatorname{gph}\pi^*
&\quad\Longleftrightarrow\quad
\text{there exists }(y,s)\text{ such that }
(\pi,y,s)\in\mathcal{K}(z,q)\\
&\quad\Longleftrightarrow\quad
\begin{gathered}
\text{there exist }\mathcal{I}\subseteq\{1,\ldots,nm\}
\text{ and }(y,s)\text{ such that }
(z,q,\pi,y,s)\in\mathcal{G}_{\mathcal{I}}
\end{gathered}\\
&\quad\Longleftrightarrow\quad
(z,q,\pi)\in
\bigcup_{\mathcal{I}\subseteq\{1,\ldots,nm\}}
\operatorname{pr}_{z,q,\pi}(\mathcal{G}_{\mathcal{I}}).
\end{aligned}
\end{equation*}
Since this equivalence holds for every such $(z,q,\pi)$, it follows that
\begin{equation*}
\operatorname{gph}\pi^*
=
\bigcup_{\mathcal{I}\subseteq\{1,\ldots,nm\}}
\operatorname{pr}_{z,q,\pi}(\mathcal{G}_{\mathcal{I}}).
\end{equation*}
Moreover, since coordinate projections are linear maps, each $\operatorname{pr}_{z,q,\pi}(\mathcal{G}_{\mathcal{I}})$ is a polyhedron \cite[Proposition~3.55(a)]{rw1998variational}. Thus, $\operatorname{gph}\pi^*$ is a finite union of polyhedra.

For each fixed $\mathcal{I}\subseteq\{1,\ldots,nm\}$, define the associated, possibly overlapping, parameter region by
\begin{equation*}
\mathcal{D}_{\mathcal{I}}
:=
\left\{(z,q)\in\mathbb{R}^{nm}\times\mathcal{Q}:
\text{there exists }(\pi,y,s)\text{ such that }
(z,q,\pi,y,s)\in\mathcal{G}_{\mathcal{I}}\right\}.
\end{equation*}
Equivalently, using the two equalities in the definition of $\mathcal{G}_{\mathcal{I}}$ to eliminate $(z,q)$ gives
\begin{equation*}
\mathcal{D}_{\mathcal{I}}
=
\left\{(\pi-A^\top y-s,A\pi):
\pi_{\mathcal{I}}\geq0,\
\pi_{\mathcal{J}}=0,\
s_{\mathcal{I}}=0,\
s_{\mathcal{J}}\geq0\right\},
\end{equation*}
where $\pi,s\in\mathbb{R}^{nm}$ and $y\in\mathbb{R}^{n+m}$. This representation shows directly that $\mathcal{D}_{\mathcal{I}}$ is nonempty for every $\mathcal{I}$; indeed, $\{(-A^\top y,0):y\in\mathbb{R}^{n+m}\}
\subseteq\mathcal{D}_{\mathcal{I}}$. Moreover, $\mathcal{D}_{\mathcal{I}}$ is the projection of $\mathcal{G}_{\mathcal{I}}$ onto the $(z,q)$-space. Since $\mathcal{G}_{\mathcal{I}}$ is a polyhedron and coordinate projections preserve polyhedrality, $\mathcal{D}_{\mathcal{I}}$ is a polyhedron. Let $\operatorname{pr}_{z,q}(z,q,\pi):=(z,q)$. Since coordinate projections commute with unions, the preceding graph representation gives
\begin{equation*}
\begin{aligned}
\operatorname{dom}\pi^*
=\operatorname{pr}_{z,q}(\operatorname{gph}\pi^*)
=\bigcup_{\mathcal{I}\subseteq\{1,\ldots,nm\}}
\operatorname{pr}_{z,q}\bigl(
\operatorname{pr}_{z,q,\pi}(\mathcal{G}_{\mathcal{I}})
\bigr)
=\bigcup_{\mathcal{I}\subseteq\{1,\ldots,nm\}}
\mathcal{D}_{\mathcal{I}}.
\end{aligned}
\end{equation*}
Since $\pi^*$ is well defined on $\mathbb{R}^{nm}\times\mathcal{Q}$, these finitely many closed polyhedra form a cover of $\operatorname{dom}\pi^*=\mathbb{R}^{nm}\times\mathcal{Q}$.

For any $\mathcal{I}\subseteq\{1,\ldots,nm\}$, let $A_{\mathcal{I}}$ denote the submatrix of $A$ formed by the columns indexed by $\mathcal{I}$. Fix $(z,q)\in\mathcal{D}_{\mathcal{I}}$. By the definition of $\mathcal{D}_{\mathcal{I}}$, there exists $(\pi,y,s)$ such that $(z,q,\pi,y,s)\in\mathcal{G}_{\mathcal{I}}$. For any such triple, the graph representation above gives $\pi=\pi^*(z,q)$, while the defining relations of $\mathcal{G}_{\mathcal{I}}$ give
\begin{equation*}
\pi_{\mathcal{J}}=0,\qquad
\pi_{\mathcal{I}}-z_{\mathcal{I}}\in\mathrm{Range}(A_{\mathcal{I}}^\top),
\qquad
A_{\mathcal{I}}\pi_{\mathcal{I}}=q.
\end{equation*}
In particular,
\begin{equation*}
q-A_{\mathcal{I}}z_{\mathcal{I}}
=A_{\mathcal{I}}(\pi_{\mathcal{I}}-z_{\mathcal{I}}).
\end{equation*}
Since $A_{\mathcal{I}}^\dagger A_{\mathcal{I}}$ is the identity on $\mathrm{Range}(A_{\mathcal{I}}^\top)$, it follows that
\begin{equation}\label{eq:parametric-projection-affine-piece}
\pi_{\mathcal{I}}
=z_{\mathcal{I}}+A_{\mathcal{I}}^\dagger(q-A_{\mathcal{I}}z_{\mathcal{I}}),
\qquad
\pi_{\mathcal{J}}=0,
\end{equation}
where $A_{\mathcal{I}}^\dagger$ denotes the Moore--Penrose inverse.
Let $F_{\mathcal{I}}$ denote the affine mapping specified by \eqref{eq:parametric-projection-affine-piece}. Thus, $\pi^*=F_{\mathcal{I}}$ on $\mathcal{D}_{\mathcal{I}}$. Since these finitely many polyhedral regions cover its domain, $\pi^*$ is piecewise affine \cite[Definition~2.47]{rw1998variational}. It is also continuous relative to its domain. Indeed, if continuity failed at some $(\bar z,\bar q)$, there would exist a sequence $(z^\nu,q^\nu)\in\mathbb{R}^{nm}\times\mathcal{Q}$ converging to $(\bar z,\bar q)$ and $\varepsilon>0$ such that
\begin{equation*}
\|\pi^*(z^\nu,q^\nu)-\pi^*(\bar z,\bar q)\|_2\geq\varepsilon
\qquad\text{for every }\nu.
\end{equation*}
Since there are only finitely many regions, after passing to a subsequence, we have $(z^\nu,q^\nu)\in\mathcal{D}_{\mathcal{I}}$ for some fixed $\mathcal{I}$ and every $\nu$. The closedness of $\mathcal{D}_{\mathcal{I}}$ gives $(\bar z,\bar q)\in\mathcal{D}_{\mathcal{I}}$. Therefore,
\begin{equation*}
\pi^*(z^\nu,q^\nu)
=F_{\mathcal{I}}(z^\nu,q^\nu)
\longrightarrow
F_{\mathcal{I}}(\bar z,\bar q)
=\pi^*(\bar z,\bar q),
\end{equation*}
because $F_{\mathcal{I}}$ is affine and hence continuous. This contradicts the preceding inequality.

Define
\begin{equation*}
\kappa:=\max_{\mathcal{I}\subseteq\{1,\ldots,nm\}}
\|A_{\mathcal{I}}^\dagger\|_2<\infty.
\end{equation*}
The maximum is finite because there are only finitely many index sets $\mathcal{I}$, and it is positive because $A$ has a nonzero column. In particular, $\kappa$ depends only on the fixed matrix $A$.

Fix $z$ and $q_1,q_2\in\mathcal{Q}$, and set $q(t):=(1-t)q_1+tq_2$ for $t\in[0,1]$. Since $\mathcal{Q}$ is convex, $(z,q(t))$ belongs to the domain of $\pi^*$ for every $t\in[0,1]$. For each $\mathcal{I}$, define
\begin{equation*}
E_{\mathcal{I}}
:=
\{t\in[0,1]:(z,q(t))\in\mathcal{D}_{\mathcal{I}}\}.
\end{equation*}
Because $\mathcal{D}_{\mathcal{I}}$ is closed and convex and $t\mapsto(z,q(t))$ is affine, $E_{\mathcal{I}}$ is either empty or a closed interval. For every nonempty $E_{\mathcal{I}}$, write
\begin{equation*}
E_{\mathcal{I}}=[\alpha_{\mathcal{I}},\beta_{\mathcal{I}}],
\qquad
0\leq\alpha_{\mathcal{I}}\leq\beta_{\mathcal{I}}\leq1,
\end{equation*}
where $\alpha_{\mathcal{I}}=\beta_{\mathcal{I}}$ is allowed. Since the regions $\mathcal{D}_{\mathcal{I}}$ cover the domain of $\pi^*$, the sets $E_{\mathcal{I}}$ cover $[0,1]$. Moreover, since there are only finitely many index sets $\mathcal{I}$, the set
\begin{equation*}
\mathcal{T}
:=
\{0,1\}
\cup
\{\alpha_{\mathcal{I}},\beta_{\mathcal{I}}:
E_{\mathcal{I}}\neq\varnothing\}
\end{equation*}
is finite. List its distinct elements in increasing order:
\begin{equation*}
0=t_0<t_1<\cdots<t_L=1.
\end{equation*}
For each $\ell\in\{1,\ldots,L\}$, choose $\tau_\ell\in(t_{\ell-1},t_\ell)$. Since the sets $E_{\mathcal{I}}$ cover $[0,1]$, there exists at least one index set
\begin{equation*}
\mathcal{I}_\ell\subseteq\{1,\ldots,nm\}
\qquad\text{such that}\qquad
\tau_\ell\in E_{\mathcal{I}_\ell}.
\end{equation*}
Fix one such $\mathcal{I}_\ell$; it is the complementarity-regime index assigned to the $\ell$th subinterval. If $\tau_\ell$ belongs to several sets $E_{\mathcal{I}}$, any one of their index sets may be chosen. Since the endpoints $\alpha_{\mathcal{I}_\ell}$ and $\beta_{\mathcal{I}_\ell}$ belong to $\mathcal{T}$ and $\mathcal{T}$ has no point in $(t_{\ell-1},t_\ell)$, the inclusion $\tau_\ell\in[\alpha_{\mathcal{I}_\ell},\beta_{\mathcal{I}_\ell}]$ implies
\begin{equation*}
\alpha_{\mathcal{I}_\ell}\leq t_{\ell-1}
\qquad\text{and}\qquad
t_\ell\leq\beta_{\mathcal{I}_\ell}.
\end{equation*}
Consequently,
\begin{equation*}
[t_{\ell-1},t_\ell]\subseteq E_{\mathcal{I}_\ell}.
\end{equation*}
Hence, $(z,q(t))\in\mathcal{D}_{\mathcal{I}_\ell}$ throughout this subinterval. Since $z$ is fixed, subtracting \eqref{eq:parametric-projection-affine-piece} at its endpoints gives
\begin{equation*}
\begin{aligned}
\bigl[\pi^*(z,q(t_\ell))-\pi^*(z,q(t_{\ell-1}))\bigr]_{\mathcal{I}_\ell}
&=
A_{\mathcal{I}_\ell}^\dagger
\bigl(q(t_\ell)-q(t_{\ell-1})\bigr),\\
\bigl[\pi^*(z,q(t_\ell))-\pi^*(z,q(t_{\ell-1}))\bigr]_{\mathcal{I}_\ell^c}
&=0.
\end{aligned}
\end{equation*}
Therefore,
\begin{equation*}
\|\pi^*(z,q(t_\ell))-\pi^*(z,q(t_{\ell-1}))\|_2
\leq
\|A_{\mathcal{I}_\ell}^\dagger\|_2
\|q(t_\ell)-q(t_{\ell-1})\|_2.
\end{equation*}
Using this estimate, the triangle inequality, and $q(t_\ell)-q(t_{\ell-1})=(t_\ell-t_{\ell-1})(q_2-q_1)$ yields
\begin{equation*}
\begin{aligned}
\|\pi^*(z,q_1)-\pi^*(z,q_2)\|_2
&\leq \sum_{\ell=1}^L
\|\pi^*(z,q(t_\ell))-\pi^*(z,q(t_{\ell-1}))\|_2\\
&\leq \sum_{\ell=1}^L
\|A_{\mathcal{I}_\ell}^\dagger\|_2
\|q(t_\ell)-q(t_{\ell-1})\|_2\\
&\leq \kappa\sum_{\ell=1}^L
(t_\ell-t_{\ell-1})\|q_2-q_1\|_2\\
&=\kappa\|q_1-q_2\|_2,
\end{aligned}
\end{equation*}
which proves \eqref{eq:parametric-projection-bound}. Notice that this argument neither requires $A$ or $A_{\mathcal{I}}$ to have full row rank nor requires the right-hand sides to lie in the relative interior of $\mathcal{Q}$. In particular, the redundant row in the transport constraints is handled by the Moore--Penrose inverses.

Now fix $z$ and $y$, and set
\begin{equation*}
\widehat{\pi}^{(z,y)}:=\mathrm{Proj}_{\mathbb{R}_+^{nm}}(A^\top y+z),
\qquad
\widehat{q}:=A\widehat{\pi}^{(z,y)}\in\mathcal{Q}.
\end{equation*}
By construction,
\begin{equation*}
\widehat{\pi}^{(z,y)}=\mathrm{vec}(\Pi(y,Z)),
\qquad
A\widehat{\pi}^{(z,y)}=\mathcal{A}\Pi(y,Z).
\end{equation*}
By the projection property,
\begin{equation*}
s:=\widehat{\pi}^{(z,y)}-z-A^\top y,\qquad
0\leq\widehat{\pi}^{(z,y)} \perp s\geq0.
\end{equation*}
Thus, $(\widehat{\pi}^{(z,y)},y,s)$ satisfies the exact KKT system for the projection problem with right-hand side $\widehat{q}$. By uniqueness of its primal solution,
\begin{equation*}
\widehat{\pi}^{(z,y)} = \pi^*(z,\widehat{q}).
\end{equation*}
Moreover, $r\in\mathcal{Q}$ and $\mathrm{vec}(\Pi^*(Z))=\pi^*(z,r)$. Applying \eqref{eq:parametric-projection-bound} with $q_1=\widehat{q}$ and $q_2=r$ gives
\begin{equation*}
\|\widehat{\pi}^{(z,y)}-\pi^*(z,r)\|_2
\leq\kappa\|\widehat{q}-r\|_2
=\kappa\|A\widehat{\pi}^{(z,y)}-r\|_2.
\end{equation*}
Since vectorization preserves the Frobenius norm, the last inequality is exactly the claimed inequality in matrix notation.
\end{proof}

\paragraph{Proof of Proposition \ref{prop:projection-error-bound}:}
Let $\kappa$ be the constant in Lemma~\ref{lem:uniform-hoffman-projection}, and set $\omega:=\kappa$.
Applying Lemma~\ref{lem:uniform-hoffman-projection} with $Z=\Pi^k-\frac{1}{\gamma}\nabla f(\Pi^k)$ and $y=y^{k+1}$ gives
\begin{equation*}
\|\Pi^{k+1}-\Pi^{k,*}\|_F
\leq \omega\|\mathcal{A}\Pi^{k+1}-r\|_2 .
\end{equation*}
The inexact condition \eqref{eq-inexact-ipg-subpro} then implies
\begin{equation*}
\|\Pi^{k+1}-\Pi^{k,*}\|_F
\leq \omega\frac{\varepsilon_k}{1+\|y^{k+1}\|_2}
\leq \omega\varepsilon_k .
\end{equation*}
This completes the proof. \hskip .5cm $\Box$

\section{Proof of Theorem \ref{thm-full-convergence}} \label{appendix-proof-full-convergence}

In this section, we establish the full sequential convergence. We first recall some standard notions in variational analysis. Let $\mathbb{E}$ be a real finite dimensional Euclidean space equipped with an inner product $\langle\cdot,\cdot\rangle$ and its induced norm $\|\cdot\|$. For a proper closed
function $h:\mathbb{E}\to\mathbb{R}\cup\{+\infty\}$, we write
$x^k\xrightarrow{h}x$ if $x^k\to x$ and $h(x^k)\to h(x)$. The limiting
subdifferential \cite[Definition~8.3]{rw1998variational} of $h$ at
$x\in\mathrm{dom}\,h$ is defined by
\begin{equation*}
\partial h(x)
:=\left\{d\in\mathbb{E}:
\exists\,x^k\xrightarrow{h}x,
~d^k\to d, ~d^k\in\widehat{\partial}h(x^k)~\text{for all}~k
\right\},
\end{equation*}
where $\widehat{\partial}h$ denotes the Fr\'echet subdifferential. Specifically,
$d\in\widehat{\partial}h(x)$ if
\begin{equation*}
\liminf_{y\ne x,\,y\to x}
\frac{h(y)-h(x)-\langle d,y-x\rangle}{\|y-x\|}
\geq 0.
\end{equation*}
When $h$ is continuously differentiable or convex, this limiting
subdifferential reduces to the classical gradient or convex subdifferential,
respectively; see \cite[Exercise~8.8 and Proposition~8.12]{rw1998variational}.

We next recall the Kurdyka-{\L}ojasiewicz (KL) property (see \cite{abs2013convergence,bdl2007the,bst2014proximal} for more details), which is now a standard technical condition for establishing the convergence of the whole sequence in the nonconvex setting. For simplicity, let $\Phi_{\nu}$ ($\nu>0$) denote a class of concave functions $\varphi:[0,\nu) \rightarrow \mathbb{R}_{+}$ satisfying: (i) $\varphi(0)=0$; (ii) $\varphi$ is continuously differentiable on $(0,\nu)$ and continuous at $0$; (iii) $\varphi'(t)>0$ for all $t\in(0,\nu)$. The KL property is described as follows.

\begin{definition}[\textbf{KL property}]\label{property-KL}
Let $h:\mathbb{E}\to\mathbb{R}\cup\{+\infty\}$ be a proper closed function.
We say that $h$ satisfies the \textbf{Kurdyka-{\L}ojasiewicz (KL)} property at
$\bar{x}\in\mathrm{dom}\,\partial h$ if there exist $\nu\in(0,+\infty]$, a
neighborhood $\mathcal{V}$ of $\bar{x}$, and a function
$\varphi\in\Phi_{\nu}$ such that
\begin{equation*}
\varphi'(h(x)-h(\bar{x}))\,\mathrm{dist}(0,\partial h(x))\geq 1
\end{equation*}
for all
$x\in\mathcal{V}\cap\{x\in\mathbb{E}: h(\bar{x})<h(x)<h(\bar{x})+\nu\}$.
The function $h$ is called a KL function if it satisfies the KL property at
every point of $\mathrm{dom}\,\partial h$.
\end{definition}

We also recall the uniformized KL property, which was established in \cite[Lemma 6]{bst2014proximal}.

\begin{proposition}[\textbf{Uniformized KL property}]\label{uniKL}
Let $h:\mathbb{E}\to\mathbb{R}\cup\{+\infty\}$ be a proper closed function, and
let $\Gamma\subseteq\mathbb{E}$ be a compact set. Suppose that $h\equiv\zeta$ on
$\Gamma$ for some constant $\zeta$ and that $h$ satisfies the KL property at
each point of $\Gamma$. Then there exist $\mu>0$, $\nu>0$, and
$\varphi\in\Phi_{\nu}$ such that
\begin{equation*}
\varphi'(h(x)-\zeta)\,\mathrm{dist}(0,\partial h(x))\geq 1
\end{equation*}
for all
$x\in\{x\in\mathbb{E}:\mathrm{dist}(x,\Gamma)<\mu\}\cap\{x\in\mathbb{E}:\zeta<h(x)<\zeta+\nu\}$.
\end{proposition}

Let $\tau>1$ be an arbitrary rational constant. The subsequent analysis makes use of the following auxiliary merit function:
\begin{equation}\label{eq:Phi-tau}
\Psi_\tau(\Pi,t)
:= f(\Pi) + \delta_{\mathcal{U}(a,b)}(\Pi) + \frac{t^\tau}{\tau},
\quad \forall\,(\Pi,t)\in\mathbb{R}^{n\times m}\times\mathbb{R}_+.
\end{equation}
For $t<0$, we define $\Psi_{\tau}(\Pi,t):=+\infty$. Since $\mathcal{U}(a,b)$ is a nonempty polytope, $f$ is quadratic, and $\tau>1$ is rational, the function $\Psi_\tau$ is proper, closed, and semialgebraic. Hence $\Psi_\tau$ satisfies the KL property; see, e.g., \cite{abrs2010proximal,abs2013convergence}.

\paragraph{Proof of Theorem \ref{thm-full-convergence}:}
\textbf{Statement 1.} From the definition of $\Pi^{k,*}$ in \eqref{defHstar} and the first-order optimality condition of the projection problem, we have that
\begin{equation}\label{eq-optimality-projection}
0 \in \partial \delta_{\mathcal{U}(a,b)}(\Pi^{k,*}) + \nabla f(\Pi^k) + \gamma (\Pi^{k,*} - \Pi^k),\quad \forall k\ge 0.
\end{equation}
Then, there exists a $D^{k,*}\in \partial \delta_{\mathcal{U}(a,b)}(\Pi^{k,*})$ such that
\begin{equation*}
D^{k,*} + \nabla f(\Pi^k) + \gamma (\Pi^{k,*} - \Pi^k) = 0, \quad \forall k\ge 0.
\end{equation*}
By the definition of the subdifferential of the indicator function $\delta_{\mathcal{U}(a,b)}$, the following inequality holds for any feasible $\Pi\in \mathcal{U}(a,b)$:
\begin{equation*}
\begin{aligned}
0\geq
&\; \langle D^{k,*}, \Pi - \Pi^{k, *}\rangle
= \langle - \nabla f(\Pi^k) - \gamma (\Pi^{k,*} - \Pi^k), \Pi - \Pi^{k, *}\rangle \\
= &\; -\langle \nabla f(\Pi^k), \Pi - \Pi^{k, *}\rangle- \gamma \langle \Pi^{k,*} - \Pi^k, \Pi - \Pi^{k, *}\rangle \\
= &\; \langle \nabla f(\Pi^k), \Pi^{k, *}-\Pi\rangle - \frac{\gamma}{2} \left(\|\Pi-\Pi^k\|_F^2 - \|\Pi-\Pi^{k,*}\|_F^2 - \|\Pi^{k,*} - \Pi^k\|_F^2\right).
\end{aligned}
\end{equation*}
Also, by the $L_f$-smoothness of $f$, we see that
\begin{equation*}
\begin{aligned}
f(\Pi^{k,*}) \leq &\; f(\Pi^k) + \langle \nabla f(\Pi^k), \Pi^{k,*}  - \Pi^k\rangle  + \frac{L_f}{2}\|\Pi^{k,*}  - \Pi^k\|_F^2, \\
f(\Pi) \geq &\; f(\Pi^k) + \langle \nabla f(\Pi^k), \Pi - \Pi^k\rangle - \frac{L_f}{2}\|\Pi - \Pi^k\|_F^2,
\end{aligned}
\end{equation*}
which implies that, for any $k\ge 0$,
\begin{equation*}
\begin{aligned}
f(\Pi^{k,*}) - f(\Pi)  \le &\; \langle \nabla f(\Pi^k), \Pi^{k,*}  - \Pi\rangle + \frac{L_f}{2}\|\Pi^{k,*}  - \Pi^k\|_F^2 + \frac{L_f}{2}\|\Pi - \Pi^k\|_F^2 \\
\le &\; \frac{\gamma}{2} \left(\|\Pi-\Pi^k\|_F^2 - \|\Pi-\Pi^{k,*}\|_F^2 - \|\Pi^{k,*} - \Pi^k\|_F^2\right)\\
&\; + \frac{L_f}{2}\|\Pi^{k,*}  - \Pi^k\|_F^2 + \frac{L_f}{2}\|\Pi - \Pi^k\|_F^2 \\
= &\; -\frac{\gamma}{2} \|\Pi-\Pi^{k,*}\|_F^2  - \frac{\gamma-L_f}{2} \|\Pi^{k,*} - \Pi^k\|_F^2 + \frac{\gamma+L_f}{2}\|\Pi-\Pi^k\|_F^2.
\end{aligned}
\end{equation*}
Now, let $\Pi = \Pi^{k-1,*}$ in the above, together with Proposition \ref{prop:projection-error-bound} and $\gamma>L_f$, we see that
\begin{equation*}
\begin{aligned}
f(\Pi^{k,*}) - f(\Pi^{k-1,*})
\leq &\; -\frac{\gamma}{2} \|\Pi^{k-1,*}-\Pi^{k,*}\|_F^2  - \frac{\gamma-L_f}{2} \|\Pi^{k,*} - \Pi^k\|_F^2 + \frac{\gamma+L_f}{2}\|\Pi^{k-1,*}-\Pi^k\|_F^2 \\
\leq &\; -\frac{\gamma-L_f}{2}\left( \|\Pi^{k-1,*}-\Pi^{k,*}\|_F^2 + \|\Pi^{k,*} - \Pi^k\|_F^2 \right) + \gamma\omega^2\varepsilon_{k-1}^2,
\end{aligned}
\end{equation*}
for all $k\geq 1$.

\textbf{Statement 2.} Since the feasible set $\mathcal{U}(a,b)$ is bounded, the sequence $\{\Pi^{k,*}\}$ is bounded. Moreover, since $\mathcal{U}(a,b)$ is compact and $f$ is continuous, $f$ is bounded below on $\mathcal{U}(a,b)$. From Statement 1 and $\gamma>L_f$, we see that
\begin{equation*}
f(\Pi^{k,*}) \leq f(\Pi^{k-1,*}) + \gamma\omega^2\varepsilon_{k-1}^2, \quad k\geq1.
\end{equation*}
This, together with $\sum_{k=0}^\infty \varepsilon_k^2 < \infty$ (due to $\sum_{k=0}^\infty \varepsilon_k < \infty$), implies that $\{f(\Pi^{k,*})\}$ is convergent; see e.g., \cite[Section 2.2]{polyak1987introduction}. This, together with the inequality in Statement 1, further implies that $\lim\limits_{k\to \infty}\|\Pi^{k,*}-\Pi^k\|_F = 0$. Consequently, $\{\Pi^{k,*}\}$ and $\{\Pi^k\}$ have the same set of accumulation points. By Theorem~\ref{thm-subsequential-convergence}, $\{\Pi^k\}$ and $\{\widetilde{\Pi}^k\}$ also share the same accumulation points, and thus the three sequences have the same set of accumulation points. Using these facts, we further have that $\lim\limits_{k\to\infty}f(\Pi^{k,*}) = f^*$, where $f^*$ is the constant given in Theorem \ref{thm-subsequential-convergence}.

\textbf{Statement 3.} In view of $\lim\limits_{k\to \infty}\|\Pi^{k,*}-\Pi^k\|_F = 0$, we only need to show that the sequence $\{\Pi^{k,*}\}$ is convergent. From Statement 1, by rearranging terms, dropping the term involving $\|\Pi^{k,*} - \Pi^k\|_F^2$ and denoting $\alpha:= (\gamma-L_f)/2,\;\beta := \gamma\omega^2>0$, we see that
\begin{equation*}
f(\Pi^{k,*}) + \beta\sum_{i=k}^\infty\varepsilon_i^2
\leq f(\Pi^{k-1,*}) + \beta\sum_{i=k-1}^\infty\varepsilon_i^2 - \alpha \|\Pi^{k,*} - \Pi^{k-1,*}\|_F^2, \quad \forall\;k\geq1.
\end{equation*}
Recall the definition of $\Psi_\tau$ in \eqref{eq:Phi-tau} and set $\sigma_k:= \tau \beta\sum_{i=k}^\infty\varepsilon_i^2$ for $k\geq0$. We see that
\begin{equation}\label{eq-decreasing-Psi}
\Psi_\tau(\Pi^{k,*}, \sigma_k^{1/\tau}) \leq \Psi_\tau(\Pi^{k-1,*}, \sigma_{k-1}^{1/\tau})- \alpha \|\Pi^{k,*} - \Pi^{k-1,*}\|_F^2,\quad \forall\;k\geq1.
\end{equation}
This implies that the sequence $\{\Psi_\tau(\Pi^{k,*}, \sigma_k^{1/\tau})\}$ is non-increasing for $k\geq1$ and it is bounded from below due to the convergence of $\{f(\Pi^{k,*})\}$. Since $\sigma_k\to 0$ as $k\to\infty$ (due to $\sum_{k=0}^\infty\varepsilon_k<\infty$), we also have that
\begin{equation*}
\lim_{k\to\infty}\Psi_\tau(\Pi^{k,*}, \sigma_k^{1/\tau})
= \lim_{k\to\infty}f(\Pi^{k,*}) = f^*.
\end{equation*}
Without loss of generality, we assume that $\Psi_\tau(\Pi^{k,*}, \sigma_k^{1/\tau}) > f^*$ for all $k\geq1$ (otherwise, $\Pi^{k,*}$ remains the same for all sufficiently large $k$).

From Statement 2, we see that for any accumulation point of the sequence $\{(\Pi^{k,*}, \sigma_k^{1/\tau})\}$, denoted as $(\Pi^*, 0)$, where $\Pi^*$ must be an accumulation point of the sequence $\{\Pi^{k,*}\}$, it must hold that $\Psi_\tau(\Pi^*,0)=f(\Pi^{*})=f^*$. Denote $\Lambda$ as the set of accumulation point of $\{(\Pi^{k,*}, \sigma_k^{1/\tau})\}$, which is nonempty and compact. Since $\Psi_\tau$ is a KL function, it then follows from the uniformized KL property (see Proposition \ref{uniKL}) that there exist $\mu>0$, $\nu>0$ and $\varphi\in\Phi_{\nu}$ such that
\begin{equation*}
\varphi'\big(\Psi_\tau(\Pi,t)-f^*\big)\,\mathrm{dist}\big(0, \partial\Psi_\tau(\Pi,t)\big)
\geq 1,
\end{equation*}
for all $(\Pi,t)$ satisfying $\mathrm{dist}((\Pi,t),\Lambda)<\mu$ and $ f^*<\Psi_\tau(\Pi,t)<f^*+\nu$. On the other hand, since $\mathrm{dist}((\Pi^{k,*}, \sigma_k^{1/\tau}), \Lambda)\to 0$ and $\{ \Psi_\tau(\Pi^{k,*}, \sigma_k^{1/\tau})\}$ is non-increasing, then for such $\mu$ and $\nu$, there exist $K\geq1$ such that $\mathrm{dist}((\Pi^{k,*}, \sigma_k^{1/\tau}), \Lambda) < \mu$ and $f^*<\Psi_\tau(\Pi^{k, *}, \sigma_k^{1/\tau})<f^*+\nu$. Consequently, it holds that
\begin{equation*}
\varphi'\big(\Psi_\tau(\Pi^{k,*},\sigma_k^{1/\tau})-f^*\big)\,\mathrm{dist}\big(0, \partial \Psi_\tau(\Pi^{k,*},\sigma_k^{1/\tau})\big) \geq 1,\quad \forall\,k \geq K.
\end{equation*}

We next derive an upper bound of $\mathrm{dist}\big(0, \partial \Psi_\tau(\Pi^{k,*},\sigma_k^{1/\tau})\big)$. First, it follows from \eqref{eq-optimality-projection} that
\begin{equation*}
W^{k,*} :=
\nabla f(\Pi^{k,*}) - \nabla f(\Pi^k) - \gamma (\Pi^{k,*} - \Pi^k)
\in \partial \delta_{\mathcal{U}(a,b)}(\Pi^{k,*}) + \nabla f(\Pi^{k,*}),
\quad \forall\,k\geq 1.
\end{equation*}
Moreover, by Proposition \ref{prop:projection-error-bound}, we have
\begin{equation*}
\begin{aligned}
\|W^{k,*}\|_F \leq &\; \|\nabla f(\Pi^{k,*}) - \nabla f(\Pi^k)\|_F + \gamma \|\Pi^{k,*} - \Pi^k\|_F
\leq (\gamma+L_f)\|\Pi^{k,*} - \Pi^k\|_F \\
\leq &\; (\gamma+L_f)\left(\|\Pi^{k,*} - \Pi^{k-1,*}\|_F + \|\Pi^{k-1,*} - \Pi^{k}\|_F\right) \\
\leq &\; 2\gamma\left(\|\Pi^{k,*} - \Pi^{k-1,*}\|_F + \omega\varepsilon_{k-1}\right).
\end{aligned}
\end{equation*}
The scalar term in $\Psi_{\tau}$, extended by $+\infty$ for $t<0$, has subdifferential $\{t^{\tau-1}\}$ for $t>0$ and $(-\infty,0]$ at $t=0$. Hence, $t^{\tau-1}$ is an admissible scalar subgradient for every $t\geq0$. Note that, for any $k\geq 1$,
\begin{equation*}
(W^{k,*}, \sigma_k^{1-1/\tau}) \in \left(\partial \delta_{\mathcal{U}(a,b)}(\Pi^{k,*}) + \nabla f(\Pi^{k,*})\right) \times\{ \sigma_k^{1-1/\tau}\} \subseteq \partial \Psi_\tau(\Pi^{k,*}, \sigma_k^{1/\tau}).
\end{equation*}
This implies that
\begin{equation*}
\mathrm{dist}(0, \partial \Psi_\tau(\Pi^{k,*}, \sigma_k^{1/\tau})) \leq \|W^{k,*}\|_F + \sigma_k^{1-1/\tau} \leq 2\gamma\left(\|\Pi^{k,*} - \Pi^{k-1,*}\|_F + \omega\varepsilon_{k-1}\right) + \sigma_k^{1-1/\tau}.
\end{equation*}

Now, we denote
\begin{equation*}
H_k:=\Psi_\tau(\Pi^{k,*},\sigma_k^{1/\tau}),\quad
\Delta_\varphi^k:= \varphi(H_k-f^*) - \varphi(H_{k+1}-f^*),\quad \forall\, k \geq K.
\end{equation*}
Since $\varphi$ is concave and $\{H_k\}$ is non-increasing and converges to $f^*$, we see that
\begin{equation*}
\Delta_\varphi^k
\geq \varphi'(H_k-f^*)(H_k-H_{k+1})
\geq \alpha\varphi'(H_k-f^*)\|\Pi^{k+1,*}-\Pi^{k,*}\|_F^2,\quad \forall\,k\geq K,
\end{equation*}
where the last inequality follows from \eqref{eq-decreasing-Psi}. Combining this with the KL inequality and the upper bound of $\mathrm{dist}(0, \partial \Psi_\tau(\Pi^{k,*}, \sigma_k^{1/\tau}))$ derived previously, we see that
\begin{equation*}
\begin{aligned}
\|\Pi^{k+1,*} - \Pi^{k,*}\|_F^2
\leq &\; \varphi'(H_k-f^*)\,\mathrm{dist}\big(0, \partial \Psi_\tau(\Pi^{k,*},\sigma_k^{1/\tau})\big)\|\Pi^{k+1,*}-\Pi^{k,*}\|_F^2  \\
\leq &\; \frac{2\gamma}{\alpha}\Delta_\varphi^k\left(\|\Pi^{k,*} - \Pi^{k-1,*}\|_F + \omega\varepsilon_{k-1} + \frac{1}{2\gamma}\sigma_k^{1-1/\tau}\right) \\
\leq &\; \left(\frac{\gamma}{\alpha}\Delta_\varphi^k + \frac{1}{2} \|\Pi^{k,*} - \Pi^{k-1,*}\|_F + \frac{\omega}{2}\varepsilon_{k-1} + \frac{1}{4\gamma}\sigma_k^{1-1/\tau}\right)^2, \quad \forall\,k\geq K.
\end{aligned}
\end{equation*}
Taking the square root of both sides of the above gives
\begin{equation*}
\|\Pi^{k+1,*} - \Pi^{k,*}\|_F \le \frac{1}{2} \|\Pi^{k,*} - \Pi^{k-1,*}\|_F + \frac{\gamma}{\alpha}\Delta_\varphi^k +  \frac{\omega}{2}\varepsilon_{k-1} + \frac{1}{4\gamma}\sigma_k^{1-1/\tau}, \quad \forall\,k\geq K,
\end{equation*}
which leads to
\begin{equation*}
\|\Pi^{k+1,*} - \Pi^{k,*}\|_F \le  \|\Pi^{k,*} - \Pi^{k-1,*}\|_F - \|\Pi^{k+1,*} - \Pi^{k,*}\|_F + \frac{2\gamma}{\alpha}\Delta_\varphi^k +  \omega\varepsilon_{k-1} + \frac{\sigma_k^{1-1/\tau}}{2\gamma}, ~~\forall\,k\geq K.
\end{equation*}
Summing the above relation from $k = K $ to $\infty$, we have
\begin{equation*}
\begin{aligned}
\sum_{k=K}^\infty  \|\Pi^{k+1,*} - \Pi^{k,*}\|_F \leq &\;  \|\Pi^{K,*} - \Pi^{K-1,*}\|_F + \frac{2\gamma}{\alpha} \varphi(H_K - f^*)
+ \omega\sum_{k = K}^\infty\varepsilon_{k-1} + \frac{1}{2\gamma} \sum_{k = K}^\infty\sigma_k^{1-1/\tau}.
\end{aligned}
\end{equation*}
This, together with $\sum_{k=0}^\infty\varepsilon_k<\infty$ and $\sum_{k=0}^\infty
\left(\sum_{i=k}^\infty\varepsilon_i^2\right)^{1-\frac{1}{\tau}}
<\infty$, implies that
\begin{equation*}
\sum_{k=K}^\infty \|\Pi^{k+1,*}-\Pi^{k,*}\|_F < \infty.
\end{equation*}
Hence, $\{\Pi^{k,*}\}$ is a Cauchy sequence and therefore converges. Since
$\|\Pi^{k,*}-\Pi^k\|_F\to0$, the whole sequence $\{\Pi^k\}$ converges to the
same limit, which is a stationary point of the GWOT problem~\eqref{eq:gw}.
This completes the proof.
\hskip .5cm $\Box$

\section{Detailed Computational Results}\label{appendix-detailed-results}

\begin{table}[htb!]
\centering
\setlength{\tabcolsep}{10pt}
\caption{Results for $n = 100$. }
\label{tab:appendix_n100}
\begin{tabular}{llccccc}
\toprule
Method & $\epsilon$ & Loss & Sparsity & Feasibility & Accuracy & Time \\
\midrule
BAPG & $10^{-3}$ & - & - & - & - & - \\
  & $10^{-2}$ & 5.97e-01 & 0.83 & 8.03e-01 & 0.01 & 0.02 \\
  & $10^{-1}$ & 3.10e-02 & 0.48 & 2.27e-02 & 0.12 & 0.50 \\
  & $10^{0}$  & 3.05e-02 & 0.00 & 1.46e-03 & 0.25 & 0.58 \\ \hline
CG   & -         & 3.60e-02 & 0.99 & 0.00e+00 & 0.02 & 0.00 \\ \hline
EPGD & $10^{-3}$ & 3.11e-02 & 0.00 & 5.54e-05 & 0.21 & 0.38 \\
  & $10^{-2}$ & 5.35e-02 & 0.00 & 5.85e-17 & 0.06 & 0.00 \\
  & $10^{-1}$ & 5.36e-02 & 0.00 & 5.83e-17 & 0.04 & 0.00 \\
  & $10^{0}$  & 5.36e-02 & 0.00 & 5.76e-17 & 0.04 & 0.00 \\ \hline
PPA  & $10^{-3}$ & 3.53e-02 & 0.97 & 4.69e-05 & 0.10 & 1.83 \\
   & $10^{-2}$ & 3.08e-02 & 0.43 & 1.04e-05 & 0.25 & 2.61 \\
   & $10^{-1}$ & 3.07e-02 & 0.01 & 1.52e-06 & 0.25 & 15.27 \\
   & $10^{0}$  & 3.07e-02 & 0.00 & 1.92e-07 & 0.24 & 38.57 \\ \hline
iPG  & -         & 3.21e-02 & 0.99 & 7.63e-09 & 0.17 & 0.92 \\
\bottomrule
\end{tabular}
\end{table}

\begin{table}[htbp]
\centering
\setlength{\tabcolsep}{10pt}
\caption{Results for $n = 200$. }
\label{tab:appendix_n200}
\begin{tabular}{llccccc}
\toprule
Method & $\epsilon$ & Loss & Sparsity & Feasibility & Accuracy & Time \\
\midrule
BAPG & $10^{-3}$ & - & - & - & - & - \\
  & $10^{-2}$ & 8.05e-01 & 0.96 & 7.18e-01 & 0.01 & 0.21 \\
  & $10^{-1}$ & 1.06e-01 & 0.23 & 4.19e-03 & 0.67 & 0.91 \\
  & $10^{0}$  & 1.10e-01 & 0.00 & 5.18e-04 & 0.55 & 1.64 \\ \hline
CG   & -         & 1.26e-01 & 0.99 & 0.00e+00 & 0.01 & 0.02 \\ \hline
EPGD & $10^{-3}$ & 1.11e-01 & 0.00 & 3.25e-05 & 0.51 & 0.49 \\
  & $10^{-2}$ & 1.43e-01 & 0.00 & 5.04e-17 & 0.03 & 0.01 \\
  & $10^{-1}$ & 1.43e-01 & 0.00 & 4.82e-17 & 0.03 & 0.01 \\
  & $10^{0}$  & 1.43e-01 & 0.00 & 5.01e-17 & 0.02 & 0.01 \\ \hline
PPA  & $10^{-3}$ & 1.16e-01 & 0.84 & 5.10e-05 & 0.39 & 3.77 \\
   & $10^{-2}$ & 1.12e-01 & 0.53 & 1.50e-05 & 0.51 & 8.97 \\
   & $10^{-1}$ & 1.10e-01 & 0.02 & 2.48e-06 & 0.55 & 22.54 \\
   & $10^{0}$  & 1.10e-01 & 0.00 & 2.67e-07 & 0.55 & 37.65 \\ \hline
iPG  & -         & 1.13e-01 & 0.99 & 1.34e-09 & 0.46 & 5.63 \\
\bottomrule
\end{tabular}
\end{table}

\begin{table}[htbp]
\centering
\setlength{\tabcolsep}{10pt}
\caption{Results for $n = 300$. }
\label{tab:appendix_n300}
\begin{tabular}{llccccc}
\toprule
Method & $\epsilon$ & Loss & Sparsity & Feasibility & Accuracy & Time \\
\midrule
BAPG & $10^{-3}$ & - & - & - & - & - \\
  & $10^{-2}$ & 1.14e+00 & 0.98 & 7.48e-01 & 0.00 & 0.53 \\
  & $10^{-1}$ & 6.03e-02 & 0.29 & 6.15e-03 & 0.51 & 3.68 \\
  & $10^{0}$  & 6.09e-02 & 0.00 & 5.35e-04 & 0.52 & 4.70 \\ \hline
CG   & -         & 8.34e-02 & 1.00 & 0.00e+00 & 0.01 & 0.05 \\ \hline
EPGD & $10^{-3}$ & 6.07e-02 & 0.00 & 1.85e-05 & 0.51 & 6.18 \\
  & $10^{-2}$ & 1.03e-01 & 0.00 & 4.50e-17 & 0.02 & 0.01 \\
  & $10^{-1}$ & 1.03e-01 & 0.00 & 4.52e-17 & 0.02 & 0.01 \\
  & $10^{0}$  & 1.03e-01 & 0.00 & 4.44e-17 & 0.02 & 0.01 \\ \hline
PPA  & $10^{-3}$ & 6.65e-02 & 0.87 & 3.07e-05 & 0.43 & 5.39 \\
   & $10^{-2}$ & 6.10e-02 & 0.49 & 6.25e-06 & 0.52 & 13.11 \\
   & $10^{-1}$ & 6.32e-02 & 0.02 & 1.09e-06 & 0.48 & 23.05 \\
   & $10^{0}$  & 6.33e-02 & 0.00 & 1.28e-07 & 0.48 & 56.78 \\ \hline
iPG  & -         & 5.85e-02 & 1.00 & 1.15e-09 & 0.57 & 15.66 \\
\bottomrule
\end{tabular}
\end{table}

\begin{table}[htbp]
\centering
\setlength{\tabcolsep}{10pt}
\caption{Results for $n = 400$. }
\label{tab:appendix_n400}
\begin{tabular}{llccccc}
\toprule
Method & $\epsilon$ & Loss & Sparsity & Feasibility & Accuracy & Time \\
\midrule
BAPG & $10^{-3}$ & - & - & - & - & - \\
  & $10^{-2}$ & 1.16e+00 & 0.99 & 7.15e-01 & 0.00 & 1.82 \\
  & $10^{-1}$ & 4.40e-02 & 0.21 & 4.05e-03 & 0.62 & 5.69 \\
  & $10^{0}$  & 4.41e-02 & 0.00 & 4.17e-04 & 0.65 & 8.28 \\ \hline
CG   & -         & 7.40e-02 & 1.00 & 0.00e+00 & 0.00 & 0.12 \\ \hline
EPGD & $10^{-3}$ & 4.44e-02 & 0.00 & 1.14e-05 & 0.62 & 8.37 \\
  & $10^{-2}$ & 9.26e-02 & 0.00 & 4.66e-17 & 0.02 & 0.03 \\
  & $10^{-1}$ & 9.26e-02 & 0.00 & 4.69e-17 & 0.01 & 0.03 \\
  & $10^{0}$  & 9.26e-02 & 0.00 & 4.60e-17 & 0.01 & 0.03 \\ \hline
PPA  & $10^{-3}$ & 4.63e-02 & 0.79 & 1.40e-04 & 0.63 & 4.89 \\
   & $10^{-2}$ & 4.44e-02 & 0.39 & 4.21e-06 & 0.64 & 13.13 \\
   & $10^{-1}$ & 4.43e-02 & 0.01 & 6.74e-07 & 0.64 & 39.81 \\
   & $10^{0}$  & 4.47e-02 & 0.00 & 1.17e-07 & 0.63 & 68.51 \\ \hline
iPG  & -         & 3.95e-02 & 1.00 & 1.44e-09 & 0.71 & 31.89 \\
\bottomrule
\end{tabular}
\end{table}

\begin{table}[htbp]
\centering
\setlength{\tabcolsep}{10pt}
\caption{Results for $n = 500$. }
\label{tab:appendix_n500}
\begin{tabular}{llccccc}
\toprule
Method & $\epsilon$ & Loss & Sparsity & Feasibility & Accuracy & Time \\
\midrule
BAPG & $10^{-3}$ & - & - & - & - & - \\
  & $10^{-2}$ & 9.09e-01 & 0.99 & 5.86e-01 & 0.00 & 7.45 \\
  & $10^{-1}$ & 4.71e-02 & 0.20 & 3.62e-03 & 0.57 & 12.33 \\
  & $10^{0}$  & 5.01e-02 & 0.00 & 4.38e-04 & 0.52 & 18.35 \\ \hline
CG   & -         & 7.52e-02 & 1.00 & 0.00e+00 & 0.00 & 0.21 \\ \hline
EPGD & $10^{-3}$ & 5.66e-02 & 0.00 & 1.59e-05 & 0.36 & 2.24 \\
  & $10^{-2}$ & 9.19e-02 & 0.00 & 4.04e-17 & 0.01 & 0.05 \\
  & $10^{-1}$ & 9.19e-02 & 0.00 & 4.22e-17 & 0.01 & 0.05 \\
  & $10^{0}$  & 9.19e-02 & 0.00 & 4.08e-17 & 0.01 & 0.05 \\ \hline
PPA  & $10^{-3}$ & 4.52e-02 & 0.76 & 3.68e-05 & 0.62 & 10.46 \\
   & $10^{-2}$ & 4.76e-02 & 0.45 & 3.78e-06 & 0.57 & 30.32 \\
   & $10^{-1}$ & 5.00e-02 & 0.01 & 6.57e-07 & 0.52 & 71.51 \\
   & $10^{0}$  & 5.31e-02 & 0.00 & 8.47e-08 & 0.48 & 100.31 \\ \hline
iPG  & -         & 3.25e-02 & 1.00 & 9.26e-10 & 0.85 & 38.41 \\
\bottomrule
\end{tabular}
\end{table}

\newpage

\section{Additional Experimental Details and Results}\label{appendix-additional-experiments}
\subsection{SNARE-seq protocol and complete results}\label{appendix-snare}
We use the paired SNARE-seq chromatin-accessibility and RNA features and cell-type labels distributed with SCOT~\cite{ChenLakeZhang2019SNARE,DemetciEtAl2020SCOT}. For each seed in $\{0,1,2,3,4\}$, we sample 250 or 500 paired cells without replacement, or retain all 1,047 cells, and independently permute the RNA ordering. The full-data runs use the same cells and differ in this hidden permutation; their standard deviations therefore do not measure variability across independently collected datasets. Pairings and labels are used only for evaluation. All solvers receive the same costs, uniform marginals, and product initialization.

Within each modality, features are normalized to unit Euclidean norm. We build an unweighted 50-nearest-neighbor graph using correlation distance, compute undirected shortest-path distances, replace infinite distances by the largest finite distance, and divide the cost matrix by its maximum. For PPA and EPGD we evaluate $\epsilon\in\{5\times10^{-4},10^{-3},10^{-2},10^{-1},1\}$; the reported BAPG configurations use $\epsilon\in\{10^{-1},1\}$. The main iPG runs use $\varepsilon_k=(k+1)^{-3}$ and a cap of 500 AGM iterations per projection. We use outer tolerance $2.5\times10^{-4}$ for iPG. The implementation checks the absolute successive-iterate Frobenius difference rather than a relative change. Solver runtimes exclude graph construction and downstream metric evaluation.

\paragraph{Alignment metrics.}
Let $X$ and $Y$ contain the row-normalized chromatin and RNA features, and let $\rho(i)$ index the RNA cell paired with chromatin cell $i$. The coupling maps chromatin into RNA feature space by $\widehat X=\operatorname{Diag}(\Pi1_n)^{-1}\Pi Y$. Define $D_{ij}=\|\widehat X_i-Y_{\rho(j)}\|_2$. We compute the bidirectional FOS score as
\begin{equation*}
\operatorname{FOS}=\frac{1}{2n(n-1)}\sum_{i=1}^n\sum_{j\ne i}
\left(\mathbf{1}_{\{D_{ij}<D_{ii}\}}+\mathbf{1}_{\{D_{ji}<D_{ii}\}}\right).
\end{equation*}
For chromatin-to-RNA LTA, a 5-nearest-neighbor classifier is trained on $\widehat X$ with the chromatin cell-type labels and evaluated on $Y$ with the RNA labels. This directional LTA is used in all SNARE-seq tables. Sparsity counts entries that are exactly zero. We evaluate all methods with the same POT squared-loss GW score. The BAPG scores are flagged because substantial marginal infeasibility prevents a direct comparison with feasible GW objective values. All 42 reported configurations returned finite results in all five runs.

\begin{table}[htbp]
\centering\small
\setlength{\tabcolsep}{6pt}
\caption{SNARE-seq results for $n=250$. Each entry gives the mean and, below it, $\pm$ one sample standard deviation over five runs. Sparsity is the percentage of exact zeros; LTA is chromatin-to-RNA label-transfer accuracy. $\dagger$: substantially infeasible BAPG couplings, whose losses are not directly comparable.}
\label{tab:snare-250}
\begin{tabular}{lrrrrrrr}
\toprule
Method & $\epsilon$ & Loss & Feasibility & Sparse (\%) & FOS & LTA & Time (s) \\
\midrule
CG & -- & \shortstack[r]{$0.02329$\\$\pm 0.00942$} & \shortstack[r]{$0$\\$\pm 0$} & \shortstack[r]{$99.60$\\$\pm 0.00$} & \shortstack[r]{$0.499$\\$\pm 0.215$} & \shortstack[r]{$0.316$\\$\pm 0.401$} & \shortstack[r]{$0.15$\\$\pm 0.11$} \\[3pt]
PPA & $5\!\times\!10^{-4}$ & \shortstack[r]{$0.02353$\\$\pm 0.00842$} & \shortstack[r]{$1.26\!\times\!10^{-3}$\\$\pm 8.95\!\times\!10^{-4}$} & \shortstack[r]{$89.55$\\$\pm 4.43$} & \shortstack[r]{$0.418$\\$\pm 0.211$} & \shortstack[r]{$0.435$\\$\pm 0.326$} & \shortstack[r]{$3.78$\\$\pm 1.52$} \\[3pt]
PPA & $10^{-3}$ & \shortstack[r]{$0.02340$\\$\pm 0.00887$} & \shortstack[r]{$5.54\!\times\!10^{-4}$\\$\pm 4.53\!\times\!10^{-4}$} & \shortstack[r]{$81.81$\\$\pm 4.82$} & \shortstack[r]{$0.405$\\$\pm 0.225$} & \shortstack[r]{$0.480$\\$\pm 0.391$} & \shortstack[r]{$3.67$\\$\pm 1.11$} \\[3pt]
PPA & $10^{-2}$ & \shortstack[r]{$0.02118$\\$\pm 0.00915$} & \shortstack[r]{$5.61\!\times\!10^{-5}$\\$\pm 2.72\!\times\!10^{-5}$} & \shortstack[r]{$61.49$\\$\pm 3.08$} & \shortstack[r]{$0.398$\\$\pm 0.213$} & \shortstack[r]{$0.524$\\$\pm 0.378$} & \shortstack[r]{$7.90$\\$\pm 1.08$} \\[3pt]
PPA & $10^{-1}$ & \shortstack[r]{$0.03200$\\$\pm 0.02556$} & \shortstack[r]{$3.33\!\times\!10^{-7}$\\$\pm 4.58\!\times\!10^{-7}$} & \shortstack[r]{$0.00$\\$\pm 0.00$} & \shortstack[r]{$0.409$\\$\pm 0.223$} & \shortstack[r]{$0.489$\\$\pm 0.396$} & \shortstack[r]{$1.46$\\$\pm 1.20$} \\[3pt]
PPA & $1$ & \shortstack[r]{$0.08500$\\$\pm 0.00603$} & \shortstack[r]{$5.22\!\times\!10^{-17}$\\$\pm 4.36\!\times\!10^{-18}$} & \shortstack[r]{$0.00$\\$\pm 0.00$} & \shortstack[r]{$0.463$\\$\pm 0.014$} & \shortstack[r]{$0.301$\\$\pm 0.173$} & \shortstack[r]{$0.00$\\$\pm 0.00$} \\[3pt]
EPGD & $5\!\times\!10^{-4}$ & \shortstack[r]{$0.02327$\\$\pm 0.00867$} & \shortstack[r]{$8.05\!\times\!10^{-4}$\\$\pm 4.51\!\times\!10^{-4}$} & \shortstack[r]{$6.92$\\$\pm 4.98$} & \shortstack[r]{$0.416$\\$\pm 0.218$} & \shortstack[r]{$0.432$\\$\pm 0.327$} & \shortstack[r]{$3.27$\\$\pm 0.79$} \\[3pt]
EPGD & $10^{-3}$ & \shortstack[r]{$0.02385$\\$\pm 0.00875$} & \shortstack[r]{$2.91\!\times\!10^{-4}$\\$\pm 2.28\!\times\!10^{-4}$} & \shortstack[r]{$0.00$\\$\pm 0.00$} & \shortstack[r]{$0.401$\\$\pm 0.234$} & \shortstack[r]{$0.487$\\$\pm 0.399$} & \shortstack[r]{$1.65$\\$\pm 0.60$} \\[3pt]
EPGD & $10^{-2}$ & \shortstack[r]{$0.02430$\\$\pm 0.00082$} & \shortstack[r]{$9.03\!\times\!10^{-10}$\\$\pm 7.79\!\times\!10^{-11}$} & \shortstack[r]{$0.00$\\$\pm 0.00$} & \shortstack[r]{$0.346$\\$\pm 0.242$} & \shortstack[r]{$0.617$\\$\pm 0.432$} & \shortstack[r]{$0.38$\\$\pm 0.24$} \\[3pt]
EPGD & $10^{-1}$ & \shortstack[r]{$0.08395$\\$\pm 0.00549$} & \shortstack[r]{$4.75\!\times\!10^{-17}$\\$\pm 1.23\!\times\!10^{-18}$} & \shortstack[r]{$0.00$\\$\pm 0.00$} & \shortstack[r]{$0.469$\\$\pm 0.127$} & \shortstack[r]{$0.417$\\$\pm 0.357$} & \shortstack[r]{$0.03$\\$\pm 0.01$} \\[3pt]
EPGD & $1$ & \shortstack[r]{$0.08500$\\$\pm 0.00603$} & \shortstack[r]{$5.42\!\times\!10^{-17}$\\$\pm 3.06\!\times\!10^{-18}$} & \shortstack[r]{$0.00$\\$\pm 0.00$} & \shortstack[r]{$0.463$\\$\pm 0.014$} & \shortstack[r]{$0.343$\\$\pm 0.101$} & \shortstack[r]{$0.00$\\$\pm 0.00$} \\[3pt]
BAPG$^{\dagger}$ & $10^{-1}$ & \shortstack[r]{$0.04440$\\$\pm 0.01617$} & \shortstack[r]{$5.53\!\times\!10^{-2}$\\$\pm 6.99\!\times\!10^{-3}$} & \shortstack[r]{$12.54$\\$\pm 7.65$} & \shortstack[r]{$0.518$\\$\pm 0.155$} & \shortstack[r]{$0.282$\\$\pm 0.289$} & \shortstack[r]{$0.55$\\$\pm 0.14$} \\[3pt]
BAPG$^{\dagger}$ & $1$ & \shortstack[r]{$0.07925$\\$\pm 0.00545$} & \shortstack[r]{$6.71\!\times\!10^{-3}$\\$\pm 5.01\!\times\!10^{-4}$} & \shortstack[r]{$0.00$\\$\pm 0.00$} & \shortstack[r]{$0.462$\\$\pm 0.054$} & \shortstack[r]{$0.192$\\$\pm 0.197$} & \shortstack[r]{$0.06$\\$\pm 0.02$} \\[3pt]
\midrule
iPG & -- & \shortstack[r]{$0.02339$\\$\pm 0.00914$} & \shortstack[r]{$3.36\!\times\!10^{-7}$\\$\pm 9.05\!\times\!10^{-8}$} & \shortstack[r]{$98.77$\\$\pm 0.24$} & \shortstack[r]{$0.421$\\$\pm 0.212$} & \shortstack[r]{$0.424$\\$\pm 0.334$} & \shortstack[r]{$7.04$\\$\pm 1.30$} \\[3pt]
\bottomrule
\end{tabular}
\end{table}

\begin{table}[htbp]
\centering\small
\setlength{\tabcolsep}{6pt}
\caption{SNARE-seq results for $n=500$. Each entry gives the mean and, below it, $\pm$ one sample standard deviation over five runs. Sparsity is the percentage of exact zeros; LTA is chromatin-to-RNA label-transfer accuracy. $\dagger$: substantially infeasible BAPG couplings, whose losses are not directly comparable.}
\label{tab:snare-500}
\begin{tabular}{lrrrrrrr}
\toprule
Method & $\epsilon$ & Loss & Feasibility & Sparse (\%) & FOS & LTA & Time (s) \\
\midrule
CG & -- & \shortstack[r]{$0.02027$\\$\pm 0.00473$} & \shortstack[r]{$0$\\$\pm 0$} & \shortstack[r]{$99.80$\\$\pm 0.00$} & \shortstack[r]{$0.366$\\$\pm 0.166$} & \shortstack[r]{$0.518$\\$\pm 0.290$} & \shortstack[r]{$0.56$\\$\pm 0.18$} \\[3pt]
PPA & $5\!\times\!10^{-4}$ & \shortstack[r]{$0.02181$\\$\pm 0.00370$} & \shortstack[r]{$1.22\!\times\!10^{-3}$\\$\pm 3.41\!\times\!10^{-4}$} & \shortstack[r]{$95.71$\\$\pm 2.60$} & \shortstack[r]{$0.356$\\$\pm 0.064$} & \shortstack[r]{$0.490$\\$\pm 0.138$} & \shortstack[r]{$24.20$\\$\pm 11.33$} \\[3pt]
PPA & $10^{-3}$ & \shortstack[r]{$0.02205$\\$\pm 0.00376$} & \shortstack[r]{$6.07\!\times\!10^{-4}$\\$\pm 1.90\!\times\!10^{-4}$} & \shortstack[r]{$88.43$\\$\pm 4.24$} & \shortstack[r]{$0.345$\\$\pm 0.063$} & \shortstack[r]{$0.518$\\$\pm 0.123$} & \shortstack[r]{$16.96$\\$\pm 8.23$} \\[3pt]
PPA & $10^{-2}$ & \shortstack[r]{$0.02017$\\$\pm 0.00481$} & \shortstack[r]{$4.86\!\times\!10^{-5}$\\$\pm 1.77\!\times\!10^{-5}$} & \shortstack[r]{$57.57$\\$\pm 4.37$} & \shortstack[r]{$0.319$\\$\pm 0.080$} & \shortstack[r]{$0.584$\\$\pm 0.163$} & \shortstack[r]{$17.93$\\$\pm 1.59$} \\[3pt]
PPA & $10^{-1}$ & \shortstack[r]{$0.08069$\\$\pm 0.00410$} & \shortstack[r]{$4.94\!\times\!10^{-17}$\\$\pm 2.08\!\times\!10^{-18}$} & \shortstack[r]{$0.00$\\$\pm 0.00$} & \shortstack[r]{$0.490$\\$\pm 0.025$} & \shortstack[r]{$0.348$\\$\pm 0.172$} & \shortstack[r]{$0.03$\\$\pm 0.00$} \\[3pt]
PPA & $1$ & \shortstack[r]{$0.08121$\\$\pm 0.00423$} & \shortstack[r]{$4.93\!\times\!10^{-17}$\\$\pm 1.20\!\times\!10^{-18}$} & \shortstack[r]{$0.00$\\$\pm 0.00$} & \shortstack[r]{$0.492$\\$\pm 0.022$} & \shortstack[r]{$0.324$\\$\pm 0.178$} & \shortstack[r]{$0.03$\\$\pm 0.00$} \\[3pt]
EPGD & $5\!\times\!10^{-4}$ & \shortstack[r]{$0.02169$\\$\pm 0.00402$} & \shortstack[r]{$4.92\!\times\!10^{-4}$\\$\pm 2.73\!\times\!10^{-4}$} & \shortstack[r]{$2.71$\\$\pm 1.97$} & \shortstack[r]{$0.333$\\$\pm 0.067$} & \shortstack[r]{$0.555$\\$\pm 0.137$} & \shortstack[r]{$7.65$\\$\pm 3.32$} \\[3pt]
EPGD & $10^{-3}$ & \shortstack[r]{$0.02256$\\$\pm 0.00414$} & \shortstack[r]{$1.40\!\times\!10^{-4}$\\$\pm 7.51\!\times\!10^{-5}$} & \shortstack[r]{$0.00$\\$\pm 0.00$} & \shortstack[r]{$0.333$\\$\pm 0.067$} & \shortstack[r]{$0.550$\\$\pm 0.140$} & \shortstack[r]{$4.44$\\$\pm 1.61$} \\[3pt]
EPGD & $10^{-2}$ & \shortstack[r]{$0.02765$\\$\pm 0.00522$} & \shortstack[r]{$8.89\!\times\!10^{-10}$\\$\pm 8.08\!\times\!10^{-11}$} & \shortstack[r]{$0.00$\\$\pm 0.00$} & \shortstack[r]{$0.302$\\$\pm 0.089$} & \shortstack[r]{$0.544$\\$\pm 0.153$} & \shortstack[r]{$0.99$\\$\pm 0.30$} \\[3pt]
EPGD & $10^{-1}$ & \shortstack[r]{$0.08069$\\$\pm 0.00410$} & \shortstack[r]{$4.88\!\times\!10^{-17}$\\$\pm 2.95\!\times\!10^{-18}$} & \shortstack[r]{$0.00$\\$\pm 0.00$} & \shortstack[r]{$0.490$\\$\pm 0.025$} & \shortstack[r]{$0.326$\\$\pm 0.152$} & \shortstack[r]{$0.03$\\$\pm 0.00$} \\[3pt]
EPGD & $1$ & \shortstack[r]{$0.08121$\\$\pm 0.00423$} & \shortstack[r]{$4.99\!\times\!10^{-17}$\\$\pm 2.08\!\times\!10^{-18}$} & \shortstack[r]{$0.00$\\$\pm 0.00$} & \shortstack[r]{$0.492$\\$\pm 0.022$} & \shortstack[r]{$0.373$\\$\pm 0.232$} & \shortstack[r]{$0.03$\\$\pm 0.01$} \\[3pt]
BAPG$^{\dagger}$ & $10^{-1}$ & \shortstack[r]{$0.04469$\\$\pm 0.03152$} & \shortstack[r]{$4.24\!\times\!10^{-2}$\\$\pm 5.69\!\times\!10^{-3}$} & \shortstack[r]{$0.00$\\$\pm 0.00$} & \shortstack[r]{$0.623$\\$\pm 0.030$} & \shortstack[r]{$0.058$\\$\pm 0.033$} & \shortstack[r]{$1.07$\\$\pm 0.24$} \\[3pt]
BAPG$^{\dagger}$ & $1$ & \shortstack[r]{$0.07555$\\$\pm 0.00350$} & \shortstack[r]{$5.06\!\times\!10^{-3}$\\$\pm 1.55\!\times\!10^{-4}$} & \shortstack[r]{$0.00$\\$\pm 0.00$} & \shortstack[r]{$0.479$\\$\pm 0.032$} & \shortstack[r]{$0.352$\\$\pm 0.209$} & \shortstack[r]{$0.12$\\$\pm 0.11$} \\[3pt]
\midrule
iPG & -- & \shortstack[r]{$0.02173$\\$\pm 0.00406$} & \shortstack[r]{$1.50\!\times\!10^{-6}$\\$\pm 3.54\!\times\!10^{-7}$} & \shortstack[r]{$98.91$\\$\pm 0.04$} & \shortstack[r]{$0.402$\\$\pm 0.153$} & \shortstack[r]{$0.442$\\$\pm 0.269$} & \shortstack[r]{$11.39$\\$\pm 4.76$} \\[3pt]
\bottomrule
\end{tabular}
\end{table}

\begin{table}[htbp]
\centering\small
\setlength{\tabcolsep}{6pt}
\caption{SNARE-seq results for $n=1047$. Each entry gives the mean and, below it, $\pm$ one sample standard deviation over five runs. Sparsity is the percentage of exact zeros; LTA is chromatin-to-RNA label-transfer accuracy. $\dagger$: substantially infeasible BAPG couplings, whose losses are not directly comparable.}
\label{tab:snare-1047}
\begin{tabular}{lrrrrrrr}
\toprule
Method & $\epsilon$ & Loss & Feasibility & Sparse (\%) & FOS & LTA & Time (s) \\
\midrule
CG & -- & \shortstack[r]{$0.03743$\\$\pm 0.00357$} & \shortstack[r]{$0$\\$\pm 0$} & \shortstack[r]{$99.90$\\$\pm 0.00$} & \shortstack[r]{$0.269$\\$\pm 0.065$} & \shortstack[r]{$0.668$\\$\pm 0.104$} & \shortstack[r]{$5.78$\\$\pm 2.18$} \\[3pt]
PPA & $5\!\times\!10^{-4}$ & \shortstack[r]{$0.03564$\\$\pm 0.00000$} & \shortstack[r]{$1.49\!\times\!10^{-3}$\\$\pm 2.18\!\times\!10^{-15}$} & \shortstack[r]{$96.96$\\$\pm 0.00$} & \shortstack[r]{$0.155$\\$\pm 0.000$} & \shortstack[r]{$0.979$\\$\pm 0.000$} & \shortstack[r]{$156.47$\\$\pm 27.02$} \\[3pt]
PPA & $10^{-3}$ & \shortstack[r]{$0.03623$\\$\pm 0.00000$} & \shortstack[r]{$6.99\!\times\!10^{-4}$\\$\pm 2.08\!\times\!10^{-15}$} & \shortstack[r]{$91.90$\\$\pm 0.00$} & \shortstack[r]{$0.228$\\$\pm 0.000$} & \shortstack[r]{$0.714$\\$\pm 0.000$} & \shortstack[r]{$134.81$\\$\pm 14.17$} \\[3pt]
PPA & $10^{-2}$ & \shortstack[r]{$0.03612$\\$\pm 0.00000$} & \shortstack[r]{$1.41\!\times\!10^{-5}$\\$\pm 3.54\!\times\!10^{-19}$} & \shortstack[r]{$43.75$\\$\pm 0.00$} & \shortstack[r]{$0.227$\\$\pm 0.000$} & \shortstack[r]{$0.728$\\$\pm 0.000$} & \shortstack[r]{$77.89$\\$\pm 10.44$} \\[3pt]
PPA & $10^{-1}$ & \shortstack[r]{$0.09145$\\$\pm 0.00000$} & \shortstack[r]{$4.66\!\times\!10^{-17}$\\$\pm 7.63\!\times\!10^{-19}$} & \shortstack[r]{$0.00$\\$\pm 0.00$} & \shortstack[r]{$0.488$\\$\pm 0.000$} & \shortstack[r]{$0.669$\\$\pm 0.000$} & \shortstack[r]{$0.20$\\$\pm 0.07$} \\[3pt]
PPA & $1$ & \shortstack[r]{$0.09187$\\$\pm 0.00000$} & \shortstack[r]{$4.74\!\times\!10^{-17}$\\$\pm 2.28\!\times\!10^{-19}$} & \shortstack[r]{$0.00$\\$\pm 0.00$} & \shortstack[r]{$0.490$\\$\pm 0.000$} & \shortstack[r]{$0.669$\\$\pm 0.000$} & \shortstack[r]{$0.17$\\$\pm 0.03$} \\[3pt]
EPGD & $5\!\times\!10^{-4}$ & \shortstack[r]{$0.03511$\\$\pm 0.00000$} & \shortstack[r]{$1.67\!\times\!10^{-4}$\\$\pm 0$} & \shortstack[r]{$2.81$\\$\pm 0.00$} & \shortstack[r]{$0.150$\\$\pm 0.000$} & \shortstack[r]{$0.985$\\$\pm 0.000$} & \shortstack[r]{$25.98$\\$\pm 3.35$} \\[3pt]
EPGD & $10^{-3}$ & \shortstack[r]{$0.03693$\\$\pm 0.00000$} & \shortstack[r]{$3.60\!\times\!10^{-5}$\\$\pm 8.49\!\times\!10^{-19}$} & \shortstack[r]{$0.00$\\$\pm 0.00$} & \shortstack[r]{$0.223$\\$\pm 0.000$} & \shortstack[r]{$0.724$\\$\pm 0.000$} & \shortstack[r]{$23.50$\\$\pm 2.12$} \\[3pt]
EPGD & $10^{-2}$ & \shortstack[r]{$0.04427$\\$\pm 0.00000$} & \shortstack[r]{$9.39\!\times\!10^{-10}$\\$\pm 5.06\!\times\!10^{-19}$} & \shortstack[r]{$0.00$\\$\pm 0.00$} & \shortstack[r]{$0.176$\\$\pm 0.000$} & \shortstack[r]{$0.847$\\$\pm 0.000$} & \shortstack[r]{$4.52$\\$\pm 0.76$} \\[3pt]
EPGD & $10^{-1}$ & \shortstack[r]{$0.09145$\\$\pm 0.00000$} & \shortstack[r]{$4.71\!\times\!10^{-17}$\\$\pm 4.56\!\times\!10^{-19}$} & \shortstack[r]{$0.00$\\$\pm 0.00$} & \shortstack[r]{$0.488$\\$\pm 0.000$} & \shortstack[r]{$0.669$\\$\pm 0.000$} & \shortstack[r]{$0.18$\\$\pm 0.06$} \\[3pt]
EPGD & $1$ & \shortstack[r]{$0.09187$\\$\pm 0.00000$} & \shortstack[r]{$4.65\!\times\!10^{-17}$\\$\pm 2.49\!\times\!10^{-19}$} & \shortstack[r]{$0.00$\\$\pm 0.00$} & \shortstack[r]{$0.490$\\$\pm 0.000$} & \shortstack[r]{$0.669$\\$\pm 0.000$} & \shortstack[r]{$0.20$\\$\pm 0.05$} \\[3pt]
BAPG$^{\dagger}$ & $10^{-1}$ & \shortstack[r]{$0.03247$\\$\pm 0.00000$} & \shortstack[r]{$3.01\!\times\!10^{-2}$\\$\pm 0$} & \shortstack[r]{$0.00$\\$\pm 0.00$} & \shortstack[r]{$0.624$\\$\pm 0.000$} & \shortstack[r]{$0.144$\\$\pm 0.000$} & \shortstack[r]{$1.41$\\$\pm 0.27$} \\[3pt]
BAPG$^{\dagger}$ & $1$ & \shortstack[r]{$0.08438$\\$\pm 0.00000$} & \shortstack[r]{$3.67\!\times\!10^{-3}$\\$\pm 0$} & \shortstack[r]{$0.00$\\$\pm 0.00$} & \shortstack[r]{$0.490$\\$\pm 0.000$} & \shortstack[r]{$0.670$\\$\pm 0.000$} & \shortstack[r]{$0.22$\\$\pm 0.04$} \\[3pt]
\midrule
iPG & -- & \shortstack[r]{$0.03504$\\$\pm 0.00000$} & \shortstack[r]{$4.81\!\times\!10^{-6}$\\$\pm 1.75\!\times\!10^{-17}$} & \shortstack[r]{$99.19$\\$\pm 0.00$} & \shortstack[r]{$0.150$\\$\pm 0.000$} & \shortstack[r]{$0.985$\\$\pm 0.000$} & \shortstack[r]{$21.28$\\$\pm 0.97$} \\[3pt]
\bottomrule
\end{tabular}
\end{table}

\subsection{Projection work and runtime breakdown}\label{appendix-projection-cost}

We solve the dual projection problem~\eqref{eq-dual-proj} with AGM, warm-started from the previous dual solution. Each inner iteration costs $O(nm)$, whereas a dense GW-gradient evaluation costs $O(n^2m+nm^2)$. For $K$ outer steps with $T_k$ inner iterations at step $k$, the arithmetic cost is $O(K(n^2m+nm^2)+nm\sum_{k=0}^{K-1}T_k)$, with $O(n^2+m^2+nm)$ storage. Instrumented runs show that projection accounts for 84.7\%--95.7\% of solver time (Table~\ref{tab:projection-cost}). The implementation caps each projection at 5,000 iterations for graph alignment and 500 for SNARE-seq. The reported cap-hit rates are 0.20\%--0.34\% for graphs and 10.38\%, 9.73\%, and 2.44\% for SNARE-seq at $n=250,500,1047$. Calls stopped before verifying~\eqref{eq-inexact-ipg-subpro} are practical safeguards and are not covered by the theoretical stopping rule. The inner-accuracy ablation in Appendix~\ref{appendix-inner-accuracy} quantifies this cost--accuracy tradeoff.

\begin{table}[htbp]
\centering\small
\setlength{\tabcolsep}{6pt}
\caption{Mean projection work and runtime shares. The cap is 5,000 AGM iterations for graphs and 500 for SNARE-seq. Percentages are computed from internal solver times.}
\label{tab:projection-cost}
\begin{tabular}{lrrrrrr}
\toprule
Dataset & $n$ & Outer calls & Inner/call & Cap (\%) & Gradient (\%) & Proj. (\%) \\
\midrule
Graph & 100 & 526 & 91.9 & 0.34 & 3.1 & 95.7 \\
Graph & 200 & 546 & 144.8 & 0.22 & 9.0 & 90.1 \\
Graph & 300 & 561 & 206.4 & 0.22 & 5.7 & 93.7 \\
Graph & 400 & 561 & 255.1 & 0.20 & 4.1 & 95.4 \\
Graph & 500 & 516 & 174.8 & 0.25 & 6.7 & 92.3 \\
\midrule
SNARE-seq & 250 & 157 & 271.3 & 10.38 & 7.1 & 90.3 \\
SNARE-seq & 500 & 73 & 227.6 & 9.73 & 6.6 & 89.6 \\
SNARE-seq & 1047 & 41 & 184.1 & 2.44 & 8.3 & 84.7 \\
\bottomrule
\end{tabular}
\end{table}

Table~\ref{tab:projection-cost} reports instrumented reruns of the synthetic graph experiments over 20 seeds and the five-run SNARE-seq experiments. Outer counts equal the number of projection calls; inner counts are averaged per call within each run and then across runs, including capped calls. Timing shares follow the recorded summaries: ratios of mean component to mean internal solver time for graph reruns, and mean per-run shares for SNARE-seq; remaining time includes setup and other solver work. These separate diagnostic reruns need not have the same absolute wall-clock times as the original graph results.

AGM uses the previous approximate dual solution as its next starting point and performs row/column potential additions, an entrywise positive-part operation, and marginal-residual evaluation. The implementation uses the stricter residual target $\varepsilon_k/(6+\|y\|_2)$, which implies~\eqref{eq-inexact-ipg-subpro} when satisfied. The reported cap-hit statistic counts calls using the full iteration budget; the recorded iteration counts do not separately certify the final residual of every capped call. Reaching a cap must therefore not be interpreted as successful verification of the theoretical criterion. Algorithm~\ref{algo-iPG} requires that criterion to hold before proceeding.

For $m=n$, the arithmetic and storage costs reduce to $O(Kn^3+n^2\sum_{k=0}^{K-1}T_k)$ and $O(n^2)$, respectively. The observed projection bottleneck is consistent with many quadratic-cost inner iterations at these sizes; dense gradient products can become limiting at larger scales. The current experiments do not establish competitiveness at $n\geq10,000$, where low-rank or structured approximations may be needed. Although a sparse returned coupling can be stored and applied to $d$-dimensional features in $O(\operatorname{nnz}(\Pi)d)$ work, the solver itself retains dense matrices.

\subsection{Sensitivity to the inner-accuracy schedule}\label{appendix-inner-accuracy}
On one fixed 250-cell SNARE-seq instance (seed 0), we vary $\varepsilon_k=(k+1)^{-\alpha}$ with all other settings held fixed. The ablation uses outer tolerance $10^{-4}$, an outer stopping check every five iterations, and a projection cap of 500. It is a single-instance diagnostic, distinct from the five-run main experiment using $\alpha=3$ and outer tolerance $2.5\times10^{-4}$.

\begin{table}[htbp]
\centering\small
\setlength{\tabcolsep}{6pt}
\caption{Single-instance inner-accuracy ablation on SNARE-seq ($n=250$). Total inner work counts all AGM iterations, including capped calls.}
\label{tab:inner-accuracy}
\begin{tabular}{rrrrrrr}
\toprule
$\alpha$ & Outer calls & Inner/call & Total inner & Cap (\%) & GW loss & Feasibility \\
\midrule
1.1 & 501 & 34.6 & 17,343 & 0.0 & 0.021158 & $2.50\!\times\!10^{-4}$ \\
2.1 & 391 & 86.2 & 33,706 & 0.0 & 0.017621 & $8.36\!\times\!10^{-7}$ \\
3.1 & 431 & 456.4 & 196,688 & 75.9 & 0.017604 & $1.15\!\times\!10^{-7}$ \\
4.1 & 361 & 486.9 & 175,760 & 93.6 & 0.017613 & $1.87\!\times\!10^{-9}$ \\
\bottomrule
\end{tabular}
\end{table}

For $\alpha\in\{2.1,3.1,4.1\}$, the final losses differ by less than 0.1\%, FOS ranges from 0.1680 to 0.1738, and chromatin-to-RNA LTA ranges from 0.972 to 0.980. Tightening the schedule mainly improves feasibility while increasing inner work and the frequency of capped calls. The $\alpha=1.1$ run has worse loss and alignment (FOS 0.6201, LTA 0.168). This finite-run comparison does not identify a universal optimal exponent or replace the theorem's stopping requirements.

\subsection{Sphere benchmarks with analytical reference values}\label{appendix-spheres}
We also compare iPG with analytical continuous $(4,2)$-GW distances between unit spheres~\cite{AryaEtAl2026Spheres}, using five independent samples of $N=50,100,200$ points per sphere. The mean relative errors range from 0.60\% to 8.32\%; for $\mathbb S^2$ versus $\mathbb S^3$, the error decreases from 3.64\% to 0.60\% as $N$ increases from 50 to 200. These errors combine sampling and optimization effects and are not certified discrete optimality gaps. Table~\ref{tab:sphere-reference} reports all results; Appendix~\ref{appendix-spheres} specifies the cost convention.

\begin{table}[htbp]
\centering\small
\setlength{\tabcolsep}{8pt}
\caption{Sphere benchmark: five-sample mean $\pm$ standard deviation. Analytical references are continuous $(4,2)$-GW distances; errors are not certified discrete optimality gaps.}
\label{tab:sphere-reference}
\begin{tabular}{lrrrrr}
\toprule
Spheres & $N$ & $\widehat d_{\mathrm{GW},4,2}$ & $d^\star_{\mathrm{GW},4,2}$ & Absolute error & Relative (\%) \\
\midrule
$\mathbb S^1,\mathbb S^2$ & 50 & $0.4587\pm 0.0175$ & $0.4823$ & $0.0237\pm 0.0175$ & $4.91\pm 3.64$ \\
$\mathbb S^1,\mathbb S^2$ & 100 & $0.4537\pm 0.0114$ & $0.4823$ & $0.0286\pm 0.0114$ & $5.94\pm 2.36$ \\
$\mathbb S^1,\mathbb S^2$ & 200 & $0.4640\pm 0.0100$ & $0.4823$ & $0.0184\pm 0.0100$ & $3.81\pm 2.07$ \\
$\mathbb S^1,\mathbb S^3$ & 50 & $0.4820\pm 0.0127$ & $0.5257$ & $0.0438\pm 0.0127$ & $8.32\pm 2.41$ \\
$\mathbb S^1,\mathbb S^3$ & 100 & $0.5007\pm 0.0125$ & $0.5257$ & $0.0250\pm 0.0125$ & $4.76\pm 2.37$ \\
$\mathbb S^1,\mathbb S^3$ & 200 & $0.5008\pm 0.0076$ & $0.5257$ & $0.0250\pm 0.0076$ & $4.75\pm 1.45$ \\
$\mathbb S^2,\mathbb S^3$ & 50 & $0.4152\pm 0.0126$ & $0.4006$ & $0.0146\pm 0.0126$ & $3.64\pm 3.15$ \\
$\mathbb S^2,\mathbb S^3$ & 100 & $0.4035\pm 0.0052$ & $0.4006$ & $0.0050\pm 0.0026$ & $1.24\pm 0.66$ \\
$\mathbb S^2,\mathbb S^3$ & 200 & $0.4007\pm 0.0039$ & $0.4006$ & $0.0024\pm 0.0028$ & $0.60\pm 0.71$ \\
\bottomrule
\end{tabular}
\end{table}

Arya et al.~\cite{AryaEtAl2026Spheres} give exact continuous $(4,2)$-GW distances between Euclidean unit spheres with uniform measures. We consider $(\mathbb S^1,\mathbb S^2)$, $(\mathbb S^1,\mathbb S^3)$, and $(\mathbb S^2,\mathbb S^3)$. For each pair and $N\in\{50,100,200\}$, we independently sample $N$ points uniformly on each sphere using normalized Gaussian vectors, with five seeds. Couplings have uniform empirical marginals. We use unnormalized squared-Euclidean costs, $(C_\mathcal{X})_{ii'}=\|x_i-x_{i'}\|_2^2$ and $(C_\mathcal{Y})_{jj'}=\|y_j-y_{j'}\|_2^2$. The full squared-loss score is converted to the distance convention in that reference by
\begin{equation*}
\widehat d_{\mathrm{GW},4,2}=\left(\frac{\mathcal E_{\mathrm{GW}}(\widehat\Pi)}{16}\right)^{1/4}.
\end{equation*}
Thus the sphere diagnostic uses squared distances as input costs; using ordinary Euclidean costs with this fourth-root conversion would target a different quantity. Table~\ref{tab:sphere-reference} reports mean and sample standard deviation over the five independent samples. Relative error is $100|\widehat d_{\mathrm{GW},4,2}-d^\star_{\mathrm{GW},4,2}|/d^\star_{\mathrm{GW},4,2}$.

For $\mathbb S^2$ versus $\mathbb S^3$, the mean relative error decreases from 3.64\% to 0.60\% between $N=50$ and $N=200$; improvement is not monotone for every sphere pair. The analytical optimum concerns continuous uniform measures, whereas each numerical problem uses empirical measures. Hence absolute error combines sampling and optimization effects, and numerical estimates can lie below the continuous reference. These values are not certified nonnegative suboptimality gaps for the discrete instances. An SDP lower bound for each empirical problem~\cite{chen2024sdpGW} would be needed for a separate certification study; no such SDP computation is reported here.

\clearpage
\section*{NeurIPS Paper Checklist}

\begin{enumerate}

\item {\bf Claims}
    \item[] Question: Do the main claims made in the abstract and introduction accurately reflect the paper's contributions and scope?
    \item[] Answer: \answerYes{} 
    \item[] Justification: Section \ref{sec-ipg} presents the theoretical results and Section \ref{sec:numerical-graph-alignment} and the Appendix present the numerical evidence.
    \item[] Guidelines:
    \begin{itemize}
        \item The answer \answerNA{} means that the abstract and introduction do not include the claims made in the paper.
        \item The abstract and/or introduction should clearly state the claims made, including the contributions made in the paper and important assumptions and limitations. A \answerNo{} or \answerNA{} answer to this question will not be perceived well by the reviewers.
        \item The claims made should match theoretical and experimental results, and reflect how much the results can be expected to generalize to other settings.
        \item It is fine to include aspirational goals as motivation as long as it is clear that these goals are not attained by the paper.
    \end{itemize}

\item {\bf Limitations}
    \item[] Question: Does the paper discuss the limitations of the work performed by the authors?
    \item[] Answer: \answerYes{} 
    \item[] Justification: Section \ref{sec-conclusion}.
    \item[] Guidelines:
    \begin{itemize}
        \item The answer \answerNA{} means that the paper has no limitation while the answer \answerNo{} means that the paper has limitations, but those are not discussed in the paper.
        \item The authors are encouraged to create a separate ``Limitations'' section in their paper.
        \item The paper should point out any strong assumptions and how robust the results are to violations of these assumptions (e.g., independence assumptions, noiseless settings, model well-specification, asymptotic approximations only holding locally). The authors should reflect on how these assumptions might be violated in practice and what the implications would be.
        \item The authors should reflect on the scope of the claims made, e.g., if the approach was only tested on a few datasets or with a few runs. In general, empirical results often depend on implicit assumptions, which should be articulated.
        \item The authors should reflect on the factors that influence the performance of the approach. For example, a facial recognition algorithm may perform poorly when image resolution is low or images are taken in low lighting. Or a speech-to-text system might not be used reliably to provide closed captions for online lectures because it fails to handle technical jargon.
        \item The authors should discuss the computational efficiency of the proposed algorithms and how they scale with dataset size.
        \item If applicable, the authors should discuss possible limitations of their approach to address problems of privacy and fairness.
        \item While the authors might fear that complete honesty about limitations might be used by reviewers as grounds for rejection, a worse outcome might be that reviewers discover limitations that aren't acknowledged in the paper. The authors should use their best judgment and recognize that individual actions in favor of transparency play an important role in developing norms that preserve the integrity of the community. Reviewers will be specifically instructed to not penalize honesty concerning limitations.
    \end{itemize}

\item {\bf Theory assumptions and proofs}
    \item[] Question: For each theoretical result, does the paper provide the full set of assumptions and a complete (and correct) proof?
    \item[] Answer: \answerYes{} 
    \item[] Justification: Section \ref{sec-ipg} presents the theoretical results and the Appendix gives detailed analysis.
    \item[] Guidelines:
    \begin{itemize}
        \item The answer \answerNA{} means that the paper does not include theoretical results.
        \item All the theorems, formulas, and proofs in the paper should be numbered and cross-referenced.
        \item All assumptions should be clearly stated or referenced in the statement of any theorems.
        \item The proofs can either appear in the main paper or the supplemental material, but if they appear in the supplemental material, the authors are encouraged to provide a short proof sketch to provide intuition.
        \item Inversely, any informal proof provided in the core of the paper should be complemented by formal proofs provided in appendix or supplemental material.
        \item Theorems and Lemmas that the proof relies upon should be properly referenced.
    \end{itemize}

    \item {\bf Experimental result reproducibility}
    \item[] Question: Does the paper fully disclose all the information needed to reproduce the main experimental results of the paper to the extent that it affects the main claims and/or conclusions of the paper (regardless of whether the code and data are provided or not)?
    \item[] Answer: \answerYes{} 
    \item[] Justification: Section \ref{sec:numerical-graph-alignment}.
    \item[] Guidelines:
    \begin{itemize}
        \item The answer \answerNA{} means that the paper does not include experiments.
        \item If the paper includes experiments, a \answerNo{} answer to this question will not be perceived well by the reviewers: Making the paper reproducible is important, regardless of whether the code and data are provided or not.
        \item If the contribution is a dataset and\slash or model, the authors should describe the steps taken to make their results reproducible or verifiable.
        \item Depending on the contribution, reproducibility can be accomplished in various ways. For example, if the contribution is a novel architecture, describing the architecture fully might suffice, or if the contribution is a specific model and empirical evaluation, it may be necessary to either make it possible for others to replicate the model with the same dataset, or provide access to the model. In general. releasing code and data is often one good way to accomplish this, but reproducibility can also be provided via detailed instructions for how to replicate the results, access to a hosted model (e.g., in the case of a large language model), releasing of a model checkpoint, or other means that are appropriate to the research performed.
        \item While NeurIPS does not require releasing code, the conference does require all submissions to provide some reasonable avenue for reproducibility, which may depend on the nature of the contribution. For example
        \begin{enumerate}
            \item If the contribution is primarily a new algorithm, the paper should make it clear how to reproduce that algorithm.
            \item If the contribution is primarily a new model architecture, the paper should describe the architecture clearly and fully.
            \item If the contribution is a new model (e.g., a large language model), then there should either be a way to access this model for reproducing the results or a way to reproduce the model (e.g., with an open-source dataset or instructions for how to construct the dataset).
            \item We recognize that reproducibility may be tricky in some cases, in which case authors are welcome to describe the particular way they provide for reproducibility. In the case of closed-source models, it may be that access to the model is limited in some way (e.g., to registered users), but it should be possible for other researchers to have some path to reproducing or verifying the results.
        \end{enumerate}
    \end{itemize}

\item {\bf Open access to data and code}
    \item[] Question: Does the paper provide open access to the data and code, with sufficient instructions to faithfully reproduce the main experimental results, as described in supplemental material?
\item[] Answer: \answerYes{} 
\item[] Justification: Code, experiment scripts, and instructions for data access and reproducing the main experimental results are publicly available at \url{https://github.com/liangling98/GWIPG}.
    \item[] Guidelines:
    \begin{itemize}
        \item The answer \answerNA{} means that paper does not include experiments requiring code.
        \item Please see the NeurIPS code and data submission guidelines (\url{https://neurips.cc/public/guides/CodeSubmissionPolicy}) for more details.
        \item While we encourage the release of code and data, we understand that this might not be possible, so \answerNo{} is an acceptable answer. Papers cannot be rejected simply for not including code, unless this is central to the contribution (e.g., for a new open-source benchmark).
        \item The instructions should contain the exact command and environment needed to run to reproduce the results. See the NeurIPS code and data submission guidelines (\url{https://neurips.cc/public/guides/CodeSubmissionPolicy}) for more details.
        \item The authors should provide instructions on data access and preparation, including how to access the raw data, preprocessed data, intermediate data, and generated data, etc.
        \item The authors should provide scripts to reproduce all experimental results for the new proposed method and baselines. If only a subset of experiments are reproducible, they should state which ones are omitted from the script and why.
        \item At submission time, to preserve anonymity, the authors should release anonymized versions (if applicable).
        \item Providing as much information as possible in supplemental material (appended to the paper) is recommended, but including URLs to data and code is permitted.
    \end{itemize}

\item {\bf Experimental setting/details}
    \item[] Question: Does the paper specify all the training and test details (e.g., data splits, hyperparameters, how they were chosen, type of optimizer) necessary to understand the results?
    \item[] Answer: \answerYes{} 
    \item[] Justification: Section \ref{sec:numerical-graph-alignment}.
    \item[] Guidelines:
    \begin{itemize}
        \item The answer \answerNA{} means that the paper does not include experiments.
        \item The experimental setting should be presented in the core of the paper to a level of detail that is necessary to appreciate the results and make sense of them.
        \item The full details can be provided either with the code, in appendix, or as supplemental material.
    \end{itemize}

\item {\bf Experiment statistical significance}
    \item[] Question: Does the paper report error bars suitably and correctly defined or other appropriate information about the statistical significance of the experiments?
    \item[] Answer: \answerYes{} 
    \item[] Justification: Section \ref{sec:numerical-graph-alignment} and the Appendix.
    \item[] Guidelines:
    \begin{itemize}
        \item The answer \answerNA{} means that the paper does not include experiments.
        \item The authors should answer \answerYes{} if the results are accompanied by error bars, confidence intervals, or statistical significance tests, at least for the experiments that support the main claims of the paper.
        \item The factors of variability that the error bars are capturing should be clearly stated (for example, train/test split, initialization, random drawing of some parameter, or overall run with given experimental conditions).
        \item The method for calculating the error bars should be explained (closed form formula, call to a library function, bootstrap, etc.)
        \item The assumptions made should be given (e.g., Normally distributed errors).
        \item It should be clear whether the error bar is the standard deviation or the standard error of the mean.
        \item It is OK to report 1-sigma error bars, but one should state it. The authors should preferably report a 2-sigma error bar than state that they have a 96\% CI, if the hypothesis of Normality of errors is not verified.
        \item For asymmetric distributions, the authors should be careful not to show in tables or figures symmetric error bars that would yield results that are out of range (e.g., negative error rates).
        \item If error bars are reported in tables or plots, the authors should explain in the text how they were calculated and reference the corresponding figures or tables in the text.
    \end{itemize}

\item {\bf Experiments compute resources}
    \item[] Question: For each experiment, does the paper provide sufficient information on the computer resources (type of compute workers, memory, time of execution) needed to reproduce the experiments?
    \item[] Answer: \answerYes{} 
    \item[] Justification: Section \ref{sec:numerical-graph-alignment} and the Appendix.
    \item[] Guidelines:
    \begin{itemize}
        \item The answer \answerNA{} means that the paper does not include experiments.
        \item The paper should indicate the type of compute workers CPU or GPU, internal cluster, or cloud provider, including relevant memory and storage.
        \item The paper should provide the amount of compute required for each of the individual experimental runs as well as estimate the total compute.
        \item The paper should disclose whether the full research project required more compute than the experiments reported in the paper (e.g., preliminary or failed experiments that didn't make it into the paper).
    \end{itemize}

\item {\bf Code of ethics}
    \item[] Question: Does the research conducted in the paper conform, in every respect, with the NeurIPS Code of Ethics \url{https://neurips.cc/public/EthicsGuidelines}?
    \item[] Answer: \answerYes{} 
    \item[] Justification: We confirm that the research conducted in the paper conform, in every respect, with the NeurIPS Code of Ethics.
    \item[] Guidelines:
    \begin{itemize}
        \item The answer \answerNA{} means that the authors have not reviewed the NeurIPS Code of Ethics.
        \item If the authors answer \answerNo, they should explain the special circumstances that require a deviation from the Code of Ethics.
        \item The authors should make sure to preserve anonymity (e.g., if there is a special consideration due to laws or regulations in their jurisdiction).
    \end{itemize}

\item {\bf Broader impacts}
    \item[] Question: Does the paper discuss both potential positive societal impacts and negative societal impacts of the work performed?
    \item[] Answer:  \answerNA{} 
    \item[] Justification: We confirm that there is no societal impact of the work performed and our work is purly a foundamental research.
    \item[] Guidelines:
    \begin{itemize}
        \item The answer \answerNA{} means that there is no societal impact of the work performed.
        \item If the authors answer \answerNA{} or \answerNo, they should explain why their work has no societal impact or why the paper does not address societal impact.
        \item Examples of negative societal impacts include potential malicious or unintended uses (e.g., disinformation, generating fake profiles, surveillance), fairness considerations (e.g., deployment of technologies that could make decisions that unfairly impact specific groups), privacy considerations, and security considerations.
        \item The conference expects that many papers will be foundational research and not tied to particular applications, let alone deployments. However, if there is a direct path to any negative applications, the authors should point it out. For example, it is legitimate to point out that an improvement in the quality of generative models could be used to generate Deepfakes for disinformation. On the other hand, it is not needed to point out that a generic algorithm for optimizing neural networks could enable people to train models that generate Deepfakes faster.
        \item The authors should consider possible harms that could arise when the technology is being used as intended and functioning correctly, harms that could arise when the technology is being used as intended but gives incorrect results, and harms following from (intentional or unintentional) misuse of the technology.
        \item If there are negative societal impacts, the authors could also discuss possible mitigation strategies (e.g., gated release of models, providing defenses in addition to attacks, mechanisms for monitoring misuse, mechanisms to monitor how a system learns from feedback over time, improving the efficiency and accessibility of ML).
    \end{itemize}

\item {\bf Safeguards}
    \item[] Question: Does the paper describe safeguards that have been put in place for responsible release of data or models that have a high risk for misuse (e.g., pre-trained language models, image generators, or scraped datasets)?
    \item[] Answer: \answerNA{} 
    \item[] Justification: We confirm that the paper poses no such risks.
    \item[] Guidelines:
    \begin{itemize}
        \item The answer \answerNA{} means that the paper poses no such risks.
        \item Released models that have a high risk for misuse or dual-use should be released with necessary safeguards to allow for controlled use of the model, for example by requiring that users adhere to usage guidelines or restrictions to access the model or implementing safety filters.
        \item Datasets that have been scraped from the Internet could pose safety risks. The authors should describe how they avoided releasing unsafe images.
        \item We recognize that providing effective safeguards is challenging, and many papers do not require this, but we encourage authors to take this into account and make a best faith effort.
    \end{itemize}

\item {\bf Licenses for existing assets}
    \item[] Question: Are the creators or original owners of assets (e.g., code, data, models), used in the paper, properly credited and are the license and terms of use explicitly mentioned and properly respected?
    \item[] Answer: \answerYes{} 
    \item[] Justification: We properly credited the code package and datgaset.
    \item[] Guidelines:
    \begin{itemize}
        \item The answer \answerNA{} means that the paper does not use existing assets.
        \item The authors should cite the original paper that produced the code package or dataset.
        \item The authors should state which version of the asset is used and, if possible, include a URL.
        \item The name of the license (e.g., CC-BY 4.0) should be included for each asset.
        \item For scraped data from a particular source (e.g., website), the copyright and terms of service of that source should be provided.
        \item If assets are released, the license, copyright information, and terms of use in the package should be provided. For popular datasets, \url{paperswithcode.com/datasets} has curated licenses for some datasets. Their licensing guide can help determine the license of a dataset.
        \item For existing datasets that are re-packaged, both the original license and the license of the derived asset (if it has changed) should be provided.
        \item If this information is not available online, the authors are encouraged to reach out to the asset's creators.
    \end{itemize}

\item {\bf New assets}
    \item[] Question: Are new assets introduced in the paper well documented and is the documentation provided alongside the assets?
\item[] Answer: \answerYes{} 
\item[] Justification: The released implementation and experiment scripts are accompanied by installation instructions, usage examples, and reproduction instructions in the public repository at \url{https://github.com/liangling98/GWIPG}.
    \item[] Guidelines:
    \begin{itemize}
        \item The answer \answerNA{} means that the paper does not release new assets.
        \item Researchers should communicate the details of the dataset\slash code\slash model as part of their submissions via structured templates. This includes details about training, license, limitations, etc.
        \item The paper should discuss whether and how consent was obtained from people whose asset is used.
        \item At submission time, remember to anonymize your assets (if applicable). You can either create an anonymized URL or include an anonymized zip file.
    \end{itemize}

\item {\bf Crowdsourcing and research with human subjects}
    \item[] Question: For crowdsourcing experiments and research with human subjects, does the paper include the full text of instructions given to participants and screenshots, if applicable, as well as details about compensation (if any)?
    \item[] Answer: \answerNA{} 
    \item[] Justification: We confirm that the paper does not involve crowdsourcing nor research with human subjects.
    \item[] Guidelines:
    \begin{itemize}
        \item The answer \answerNA{} means that the paper does not involve crowdsourcing nor research with human subjects.
        \item Including this information in the supplemental material is fine, but if the main contribution of the paper involves human subjects, then as much detail as possible should be included in the main paper.
        \item According to the NeurIPS Code of Ethics, workers involved in data collection, curation, or other labor should be paid at least the minimum wage in the country of the data collector.
    \end{itemize}

\item {\bf Institutional review board (IRB) approvals or equivalent for research with human subjects}
    \item[] Question: Does the paper describe potential risks incurred by study participants, whether such risks were disclosed to the subjects, and whether Institutional Review Board (IRB) approvals (or an equivalent approval/review based on the requirements of your country or institution) were obtained?
    \item[] Answer: \answerNA{} 
    \item[] Justification: We confirm that the paper does not involve crowdsourcing nor research with human subjects.
    \item[] Guidelines:
    \begin{itemize}
        \item The answer \answerNA{} means that the paper does not involve crowdsourcing nor research with human subjects.
        \item Depending on the country in which research is conducted, IRB approval (or equivalent) may be required for any human subjects research. If you obtained IRB approval, you should clearly state this in the paper.
        \item We recognize that the procedures for this may vary significantly between institutions and locations, and we expect authors to adhere to the NeurIPS Code of Ethics and the guidelines for their institution.
        \item For initial submissions, do not include any information that would break anonymity (if applicable), such as the institution conducting the review.
    \end{itemize}

\item {\bf Declaration of LLM usage}
    \item[] Question: Does the paper describe the usage of LLMs if it is an important, original, or non-standard component of the core methods in this research? Note that if the LLM is used only for writing, editing, or formatting purposes and does \emph{not} impact the core methodology, scientific rigor, or originality of the research, declaration is not required.
    \item[] Answer: \answerNA{} 
    \item[] Justification: We confirm that the core method development in this research does not involve LLMs as any important, original, or non-standard components.
    \item[] Guidelines:
    \begin{itemize}
        \item The answer \answerNA{} means that the core method development in this research does not involve LLMs as any important, original, or non-standard components.
        \item Please refer to our LLM policy in the NeurIPS handbook for what should or should not be described.
    \end{itemize}

\end{enumerate}

\end{document}